\documentclass{elsarticle}

\usepackage{amssymb}
\usepackage{soul}
\usepackage{lineno,hyperref}
\usepackage{graphicx}
\usepackage{multirow}
\usepackage{mathtools}
\usepackage{makecell}
\usepackage{tabu}
\usepackage{amsmath}
\usepackage{cleveref}
\usepackage{relsize}
\usepackage{listings}
\usepackage{varwidth}
\usepackage{color}
\usepackage{bm}
\usepackage{algpseudocode}
\usepackage{algorithm}
\usepackage{subfigure}
\usepackage{comment}

\usepackage{microtype}
\usepackage{subfigure}
\usepackage{booktabs} 
\usepackage{amsthm}
\newtheorem{observation}{\textbf{Observation}}

\newtheorem{theorem}{Theorem}[section]

\newtheorem{lemma}[theorem]{Lemma}
\DeclareMathOperator*{\argmax}{arg\,max}

\definecolor{lightest_pink}{rgb}{1.0, 0.6, 1.0}
\definecolor{lighter_pink}{rgb}{1.0, 0.3, 1.0}
\definecolor{pink}{rgb}{1.0, 0.0, 1.0}

\journal{Artificial Intelligence Journal}

\begin{document}

\begin{frontmatter}



\title{MAIRE - A Model-Agnostic Interpretable Rule Extraction Procedure for Explaining Classifiers}



\author{Rajat Sharma} 
\ead{2015csb1026@iitrpr.ac.in}

\author{Nikhil Reddy }
\ead{2018csm1011@iitrpr.ac.in}

\author{Vidhya Kamakshi }
\ead{2017csz0005@iitrpr.ac.in}

\author{Narayanan C Krishnan \corref{corresponding}} 
\ead{ckn@iitrpr.ac.in}
\cortext[corresponding]{Corresponding Author}

\author{Shweta Jain} 
\ead{shwetajain@iitrpr.ac.in}

\address{Department of Computer Science and Engineering, 
Indian Institute of Technology Ropar,
Rupnagar,
Punjab - 140001,
India
}

\begin{abstract}
The paper introduces a novel framework for extracting model-agnostic human interpretable rules to explain a classifier's output. The human interpretable rule is defined as an axis-aligned hyper-cuboid containing the instance for which the classification decision has to be explained. The proposed procedure finds the largest (high \textit{coverage}) axis-aligned hyper-cuboid such that a high percentage of the instances in the hyper-cuboid have the same class label as the instance being explained (high \textit{precision}). Novel approximations to the coverage and precision measures in terms of the parameters of the hyper-cuboid are defined. They are maximized using gradient-based optimizers. The quality of the approximations is rigorously analyzed theoretically and experimentally. Heuristics for simplifying the generated explanations for achieving better interpretability and a greedy selection algorithm that combines the local explanations for creating global explanations for the model covering a large part of the instance space are also proposed. The framework is model agnostic, can be applied to any arbitrary classifier, and all types of attributes (including continuous, ordered, and unordered discrete). The wide-scale applicability of the framework is validated on a variety of synthetic and real-world datasets from different domains (tabular, text, and image).
\end{abstract}



\begin{keyword}
Interpretable Machine Learning \sep Explainable Models \sep Rule Based Explanations 


\end{keyword}

\end{frontmatter}



\section{Introduction}

The working of classic machine learning models such as simple decision trees, linear regression can be easily interpreted by analyzing the parameters of the model. But, for want of higher accuracy or better generalization performance, complex classifiers such as deep neural networks, support vector machines, and decision forests are being employed. However, improved performance comes at the cost of reduced human interpretability. Recent research focuses on explaining the working of these complex black-box models, thereby bridging the accuracy-interpretability trade-off and making them useful and trustworthy for a wider community.

An explainable approach that can work irrespective of the underlying black-box model is desirable. Such approaches are referred to as model agnostic approaches in the literature \cite{Ribeiro:2016:WIT:2939672.2939778,AAAI1816982,lore,muse}. Short and concise rules are highly human interpretable \cite{hima_lakkaraju_kdd_2016_ids,lore}. Hence we would like to develop a model-agnostic explainable approach that is capable of providing the human interpretable rules for any black-box machine learning model. A major challenge in designing such approaches lies in preserving \textit{faithfulness} to the black-box. An explanation method is said to be faithful to a black-box model if it identifies features that are truly important for the working of the model. 


A popular method of explaining the black-box model's working is by assigning ranks to the features relative to the importance the black-box model gives to a feature. This feature rank is easy to understand but is not complete. The ranking approach does not capture the class discriminative information based on the range of values. In other words, if a specific range of values for a feature results in classification to a class and outside the range corresponds to a different class, such a mechanism would not be revealed by feature ranking approaches \cite{chen2018learning,lrp}. For an explanation based on feature ranking approach to be complete, we need a measure to say how relevant is the feature value to a particular prediction. The sensitivity of the changes to the output of the model due to small changes in the feature values must also be captured. In a nutshell, feature ranking by itself is an incomplete explanation. Various factors like the importance of a feature, tolerable range of values to get the same prediction, the influence of a feature value in driving towards a prediction, is to be additionally considered along with the feature rank to provide a complete explanation to the working of the black-box model.


Another class of methods, called rule-based methods, provide intuitive explanations. Decision trees, decision lists that provide rules in the form of if-then-else statements in a hierarchical fashion, come under this category. These are global explanation models that aim to explain the working of the model in the whole instance space. Though the explanations are intuitive, it is not always simple to comprehend. If the hierarchical structure of if-then-else statements grows into longer chains, it is difficult to comprehend, and the interpretability suffers \cite{hima_lakkaraju_kdd_2016_ids}. It is to be noted that these methods partition the instance space based on the attribute values. Longer chains of if-then-else statements would mean small partitions created in the instance space. This further complicates the scenario as the rules are less generalizable.


Anchors \cite{AAAI1816982}, a recent approach overcomes the limitations of feature ranking and rule-based approaches. It builds on the observation that a complex tree of rules encompasses many simple trees. Hence instead of attempting to build a tree that spans the entire instance space and provides an explanation of the black-box model globally, it is better to provide a local explanation spanning a smaller partition of the whole instance space. The precision and coverage metrics defined in \cite{AAAI1816982} help to preserve the desired properties of posthoc explanations. But a limitation of the Anchors approach is that it is applicable only for discrete attributed datasets. In the case of continuous-valued attributes, Anchors can be applied after the continuous values are mapped onto a discrete values set only. 

Binning is employed to discretize the continuous-valued attributes by identifying a threshold to create bins. The binning threshold plays a crucial role and it may not be always possible to obtain the tighter bounds on the range of tolerable values for a prediction. Thus, an unsuitable binning threshold may lead to loss of subtle class discriminant information that may otherwise be present in the original continuous-valued attributes.

The proposed framework MAIRE is a non-trivial extension of Anchors that is applicable across any attribute type - continuous, discrete (ordered or unordered). A sample explanation generated from our approach is shown in the table ~\ref{tab:my-table}.

\begin{table}[]
\centering
\resizebox{\columnwidth}{!}{%
\begin{tabular}{cclcc}
\hline
 & \textbf{If} & \textbf{Predict} & \textbf{Coverage} & \textbf{Precision} \\ \hline
 \rotatebox[origin=c]{90}{\textbf{\,\,\,\,adult\,\,\,\,\,}} &
      \begin{tabular}[c]{@{}l@{}} 17\,\, \,$<$ \,Age\,\,\,\, \,\, \,\, \,\, \,\, \,\, $\leq$\, 43\\\,\, \,\, \,\,\,\, \,\, \,Education \,\, \,\, \,=\, High School grad\\ 0.00 \textless\,Capital-Loss \,$\leq$\,1291.44\end{tabular}
 &
  $\leq$ 50K &
  0.35 &
  0.95 \\ \hline
\rotatebox[origin=c]{90}{\textbf{\,\,\,\,abalone\,\,\,\,}} &
  \begin{tabular}[c]{@{}l@{}}\hspace*{-.5cm}\,Sex\,\,\,\,\,\,\,\,\,\,\,\,\,\,\,\,\,\,\,\,\, =\, F\\\hspace*{-1.5cm}0.07 $<$ Length \,\,\,\,\,\,\,\,\,\,\,$\leq$\,\,0.48\\ \hspace*{-1.5cm}0.05 \textless \,Diameter\,\,\,\,\,\,\, $\leq$\,0.37\end{tabular} &
  $\leq$9 Rings &
  0.25 &
  0.94 \\ \hline
\rotatebox[origin=c]{90}{\textbf{\,\,\,\,german credit\,\,\,\,}} &
  \begin{tabular}[c]{@{}l@{}}\hspace*{-3cm}\,\, \,\,\,\,\,\,\,\,\,\,\,\,\,\,\,\,\,\,\,\,\,\,\,\,\,\,\,\,\,\,\,\,\,\,\,\,\,\,\,Housing \,\,\,\,\,\,\,\,\,\,\,=\,\, own\\ \hspace*{-2.2cm}\,\,\,\,\,\,\,\,\,\,\,\,20\,\, $<$\, Duration\,\,\,\,\,\,\,\,\, $\leq$\,\,25\\\hspace*{-2cm}\,\,\,\,\,\,\,\,38 \,\,\,\textless{}\, Age \,\,\,\,\,\,\,\,\,\,\,\,\,\,\,\,\,\,\,\,$\leq$ 54\end{tabular} &
  \,\,\,good &
  0.28 &
  0.94 \\ \hline
\end{tabular}
}
\caption{Example MAIRE explanations obtained for the Adult, Abalone and German credit datasets}
\label{tab:my-table}
\end{table}

\section{Related Work}
There are significant efforts in different directions towards improving the explainability aspect of machine learning models. We broadly categorize the approaches into three.

\textbf{Model agnostic methods} are like `meta-learning' approaches that are capable of explaining the behavior of any black-box classifier. LIME\cite{Ribeiro:2016:WIT:2939672.2939778} approximates the working of the black box classifier in a local neighborhood by fitting a linear model on the black box predictions for the neighbors.  Anchors\cite{AAAI1816982} finds the decision rule for black-box model prediction such that the rule anchors the prediction adequately as governed by the precision and coverage metrics. Both LIME and Anchor generate global explanations but apply only to discrete-valued datasets. MAIRE overcomes this constraint by its novel optimization framework. \\
Learning To Explain (L2X) \cite{chen2018learning} does instance wise feature selection by maximizing the mutual information between the subset of features and the target variable. SHAP \cite{lundberg2017unified} uses Shapley values to predict the importance of features towards a prediction. Both L2X and SHAP use feature ranking approach, which is accurate in text classification. However, feature ranking may not always be optimal as feature values may play an important role in distinguishing between two classes. MAIRE, on the other hand, explains in terms of a range of values of an attribute. LORE \cite{lore} explains the black-box model by extracting rules using a decision tree applicable in a local neighborhood generated by a genetic algorithm. This method is applicable for low dimensional datasets only as with high dimensional datasets; the decision tree may grow complex, reducing interpretability. LLORE \cite{llore} uses an autoencoder to perform dimensionality reduction so that LORE \cite{lore} can be applied in the reduced dimensional space. Dimensionality reduction may lead to loss of information and should be avoided.  Further, LLORE \cite{llore} can be used only for images and an extension to handle text data has only been mentioned as a future possibility. Our proposed approach MAIRE can be applied to different domains (text, tabular, image) and does not require any modification to the dimensionality of the feature space, thus preserving all the information.

The proposed work is close to that of Lakkaraju et al \cite{muse} in the broader sense from the perspective of explanation generation in terms of rules as per attribute ranges. But their explanation generation algorithm requires value ranges to be provided, and explanation is in terms of the rules explaining how the black-box model works in the subspace defined by the given attribute values. This flexibility may be beneficial for the tabular datasets, where the value range for attributes shall make sense to end-users. While the approach in \cite{muse} is model agnostic like MAIRE, the extensive experimentation has been carried out only on tabular datasets. This need to give attribute value ranges for explanation generation is challenging in case of images or textual datasets where the attributes the black-box model works on may be different from how humans perceive the data. Our proposed approach MAIRE does not have this requirement and hence is readily applied to explain black-box models working on data from different domains.

\textbf{Model specific explainable methods} are designed to explain the working of a single or a class of models. Approaches like Guided Back Propagation \cite{guided_back_prop}, CAM \cite{cam}, and its extensions \cite{grad_cam, grad_cam_plus_plus,score_cam,ablation_cam} are applicable to architectures involving Convolutional Neural Networks only. Specifically for deep learning architectures, an attribution based technique, DeepLIFT \cite{deep_lift} provides a set of rules to assign contribution scores to every unit of a deep neural network. In contrast, the MAIRE framework explains the output of any black-box model.

\textbf{Models explainable by design} consist of methods that propose new explainable classifiers that are trained from scratch. Interpretable CNN \cite{interpretable_cnn} uses mutual information to learn interpretable parts that are filtered through predefined templates. A self-explaining architecture involving an autoencoder that determines representative prototypes clustered around inputs in a latent space was proposed by Li et al \cite{aaai_2018_prototype_vector}. In another approach, the convolutional layer feature maps are used as latent representations that helps to localize regions of the image that are similar to the prototypes \cite{this_looks_like_that,hierarchical_prototypes}. Models that explain the output in terms of human interpretable rules have also been proposed \cite{hima_lakkaraju_kdd_2016_ids,rudin_rule_ebd}. However, in contrast to the MAIRE framework, these models cannot be applied to an already deployed model.

\section{Methodology}
\subsection{Problem Statement}
Let $\{\textbf{x}_n,y_n\}_{n=1}^N$ be a set of $N$ training examples, where $\textbf{x}_n\in \mathbb{R}^D$ is a data point and $y_n \in \mathcal{Y}$ is the associated label. For simplicity, let us assume that all the attributes are continuous and are normalized to the range $[0, 1]$ and that the classification task is binary. The MAIRE framework can be easily extended for discrete attributes and multi-class classification. Given a query data point $\textbf{x}'_q$ and a classifier $f:\mathbb{R}^D\to \mathcal{Y}$, our objective is to explain the decision of the classifier at $\textbf{x}'_q$ i. e. $f(\textbf{x}'_q)$. Prior literature suggests that explanations in the form of rules defined on the values of the attributes are human interpretable \cite{AAAI1816982,hima_lakkaraju_kdd_2016_ids}. A simple way to define these rules for continuous attributes is in terms of range on the values. Thus, we define an explanation as \textbf{E} = \{\textbf{l}, \textbf{u}\}, where $\textbf{l, u} \in \mathbb{R}^D$ represent the lower and upper bounds of intervals such that $l_i\leq x'_{qi}\leq u_i,\ \forall i \in \{1,\ldots,d\}$. 

The Cartesian product of these intervals represents a hyper-cuboid denoted by $S(\textbf{l},\textbf{u})$. This is illustrated as a rectangle for the 2D dataset presented in Figure \ref{fig:illustration}. Our objective is to find an explanation that has high coverage and satisfying a certain threshold on precision. Coverage of an explanation $\textbf{E}$, $Cov(\textbf{l}, \textbf{u})$, is defined as the fraction of data points that lie within the hyper-cuboid,
\begin{equation}
    Cov(\textbf{l}, \textbf{u}) =  \frac{1}{N}\sum_{i = 1}^N \bm{1}(\textbf{x}_i \in S(\textbf{l},\textbf{u}))
    \label{eq:cov}
\end{equation}
where $\bm{1}(A)$ is the indicator function that takes the value 1 if the argument $A$ is true. High coverage means that more data points are explained using the hyper-cuboid.

Precision, $Pre(\textbf{l}, \textbf{u})$, is defined as the fraction of training instances that lie within the hyper-cuboid representing the explanation $\textbf{E}$ and whose classifier predictions match with the classifier prediction of the query point,
\begin{equation}
    Pre(\textbf{l}, \textbf{u})=\frac{\sum_{i=1}^N\bm{1} (f(\textbf{x}_i) = f(\textbf{x}'_q) \text{ and } \textbf{x}_i \in S(\textbf{l},\textbf{u}))}{\sum_{i=1}^N\bm{1}(\textbf{x}_i \in S(\textbf{l},\textbf{u}))}
    \label{eq:pre}
\end{equation}
The MAIRE framework allows for a user to define a minimum value $P$ for the precision of an explanation $Pre(\textbf{l}, \textbf{u})$. Thus the overall objective of the framework is to find an explanation (or the hyper-cuboid) that maximizes the coverage, while ensuring that the precision of the estimated explanation does not fall below the threshold $P$ i.e., 
\begin{equation}
 \underset{\{\textbf{l}, \textbf{u}\}}{\text{argmax}} \hspace{0.3cm} Cov(\textbf{l}, \textbf{u})  \hspace{0.5cm}\text{s.t. } \hspace{0.3cm} Pre(\textbf{l}, \textbf{u}) \geq P.
\label{eq:combopt}
\end{equation}
The above optimization problem is challenging to solve due to the involvement of the indicator function. For a binary classification setting, with $P = 1$, this problem becomes the bichromatic rectangle problem, a widely studied combinatorial problem in computational geometry. Bichromatic rectangle problem involves computing a rectangle containing maximum number of red points and no blue points amongst the set containing $n$ red points and $m$ blue points in $d-$dimensional space. Most of the results in this area hold for $2D$ \cite{acharyya2019,armaselu2017}. Further, the problem is NP-hard for arbitrary dimension \cite{eckstein2002}. Solving the above problem even approximately is therefore important for many applications.

We propose a novel method to transform the coverage and precision functions into differentiable approximations (with non-zero gradients), thereby making it easier to optimize using gradient-based methods.

\begin{figure}
\centering
\subfigure[Explanation in 2D]{
\includegraphics[width=0.4\textwidth]{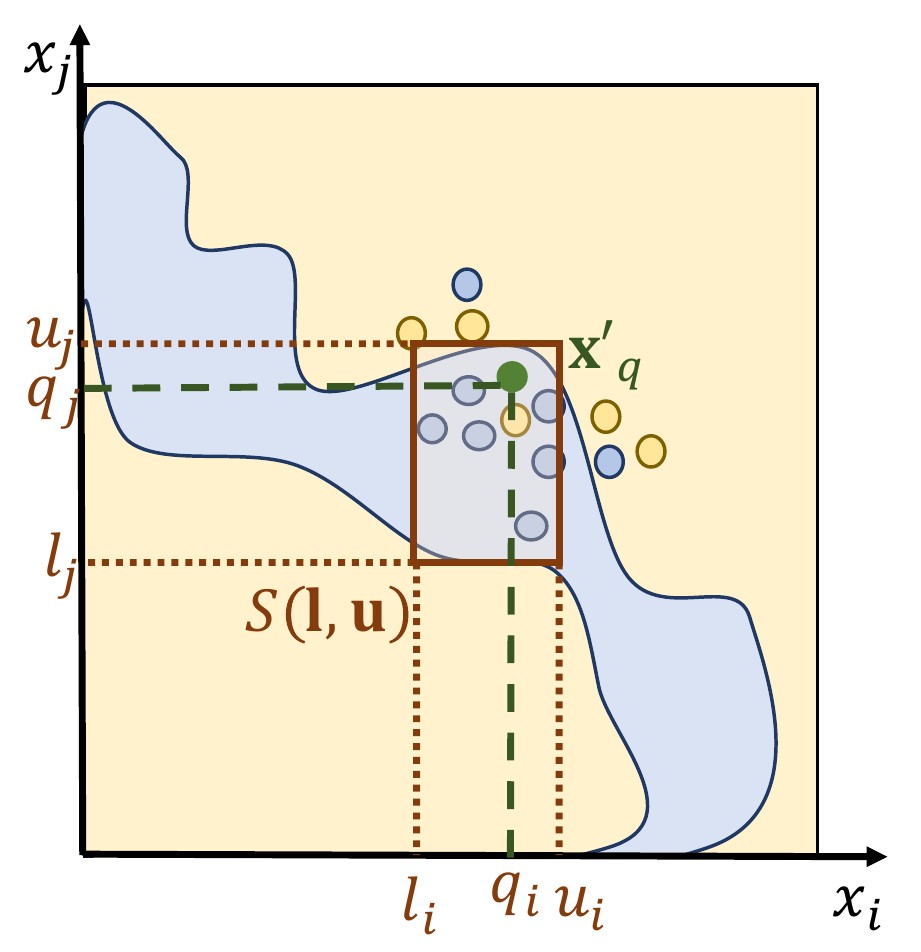}
\label{fig:illustration}}\hspace{0.02\textwidth}
\subfigure[Indicator function $\bm{1}(x)$ and its approximation $\Gamma(x)$]{
\includegraphics[width=0.4\textwidth]{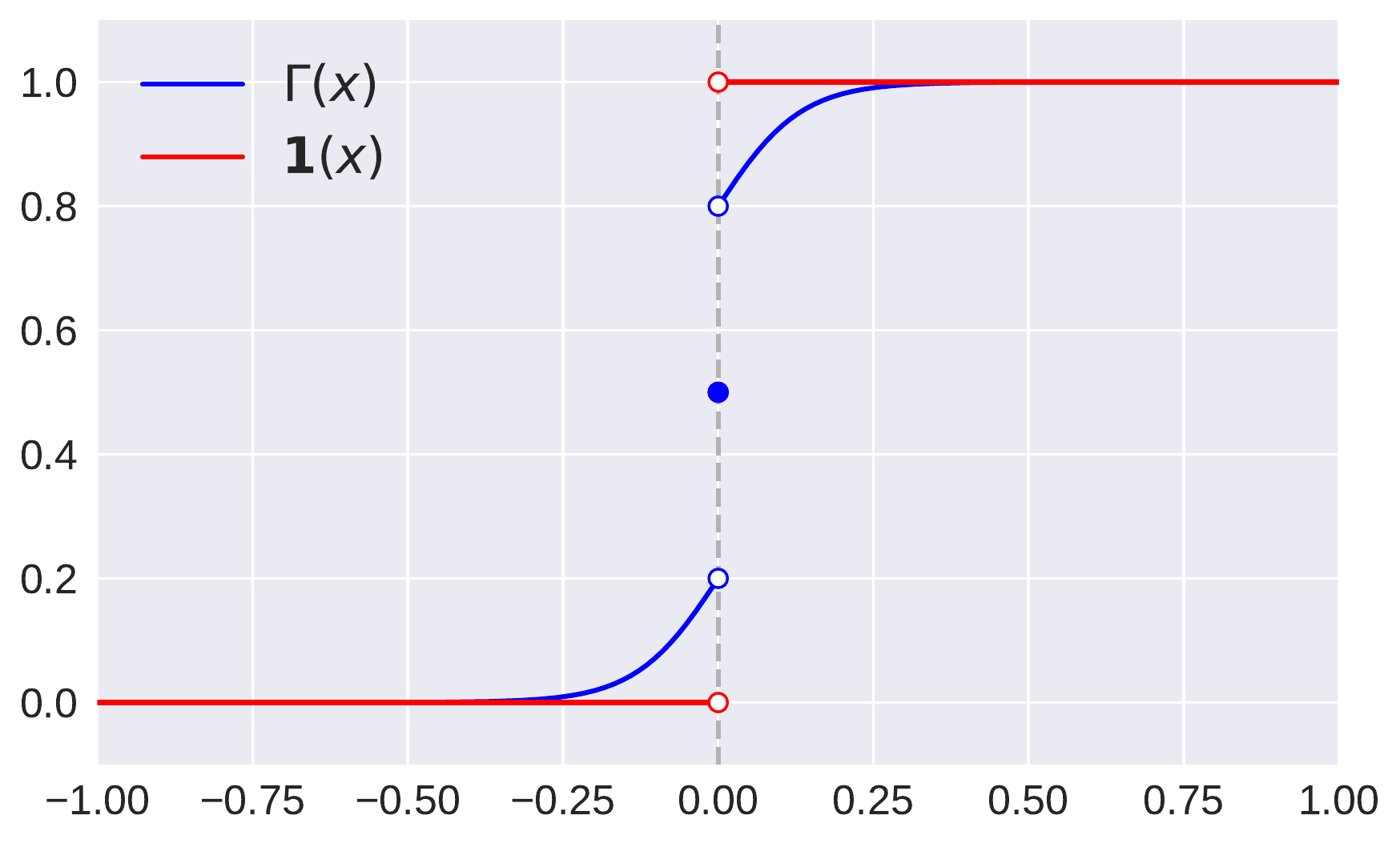}
\label{fig:gamma}}
\caption{[Best viewed in color] Illustration of the explanation and the approximation to the indicator function}
\end{figure}
\subsection{Approximations to Coverage and Precision}
We first define the function $\Gamma$, which is an approximation to the indicator function, as
$\Gamma(z) = c_1\sigma(c_2 z)+c_3(\text{sgn}(z)c_4+c_5$
where $c_1, c_2, c_3, c_4$, and $c_5$ are constants that determine the quality of the approximation, $\sigma$ is the Sigmoid function, and $\text{sgn}(z)$ is the Signum function. The constant $c_1$ is chosen to scale down the sigmoid function so that, $c_1 \sigma(c_2 z)$ takes values in a small range (effectively modeling the horizontal lines of the indicator function, but still retaining non-zero gradients). The constant $c_2$ has a high value to model the steep increase at $z = 0$ while making $\sigma(c_2 z)$ flatter for $z < 0$ and $z > 0$. The constants
$c_4$ and $c_5$ are chosen so that $(sgn(z)c_4 + c_5)$  is 1 when $z > 0$, 0.5 when $x=0$ and 0 otherwise.  $c_3$ is chosen to provide a step at $z = 0$ such that $\Gamma(z) \in (0, 1)$. This makes $\Gamma(z)$ piece-wise differentiable with non-zero gradients. The behavior of $\Gamma(z)$ is illustrated in Figure \ref{fig:gamma}. Note that, $\Gamma(z)=c_1\sigma(c_2 z)$ if $z<0$, $\Gamma(z)=c_1\sigma(c_2 z)+c_3$ if $z>0$, and $\Gamma(z) = 0.5$ if $z=0$. 

We can now approximate $\bm{1}(x_1>x_2)$ as $G(x_1, x_2)  = \Gamma(x_1-x_2)$ and $\bm{1}(x_1\geq x_2)$ as $GE(x_1,x_2) = \Gamma(x_1-x_2+c_l)$, where $c_l$ is a constant that has a low value (close to 0). The approximation to the function $\bm{1}(x_1 \text{ and } x_2 \ldots \text{ and } x_m)$ for the logical operator `and' is defined as $A(x_1,x_2,\ldots,x_m)=\Gamma(\frac{1}{m}\sum_{i=1}^m x_i -c_h)$, where $c_h$ is a constant that has a high value (close to 1). 

Let us define the functions $a_{2 j - 1}(\textbf{l}, \textbf{u}, \textbf{x}) = G(x_j, l_j) \text{ and }
a_{2 j}(\textbf{l}, \textbf{u}, \textbf{x}) = GE(u_j, x_j)$ with $j\in\{1,\ldots,D\}$. Then the approximation to the indicator function $\bm{1}(\textbf{x} \in S(\textbf{l},\textbf{u}))$  can be defined as
$h(\textbf{l}, \textbf{u}, \textbf{x}) = A(a_1(\textbf{l}, \textbf{u}, \textbf{x}), a_2(\textbf{l}, \textbf{u}, \textbf{x}),\ldots,  a_{2D}(\textbf{l}, \textbf{u}, \textbf{x}))$
Note that $h(\textbf{l}, \textbf{u}, \textbf{x})$ should take a value close to 1 if the point $\textbf{x}$ lies inside the hyper-cuboid $S(\textbf{l}, \textbf{u})$, else should take a value close to 0. We can now define approximate coverage and approximate precision as:
\begin{align*}
\hat{Cov}(\textbf{l}, \textbf{u}) = \frac{1}{N}\sum_{i=1}^N h(\textbf{l}, \textbf{u}, \textbf{x}_i)
\end{align*}
\begin{align*}
    \hat{Pre}(\textbf{l}, \textbf{u}) = \frac{\sum_{i = 1}^N h(\textbf{l}, \textbf{u}, \textbf{x}_i) (1 - (f(\textbf{x}_i) - f(\textbf{x}'_q))^2)}{\sum_{i = 1}^N h(\textbf{l}, \textbf{u}, \textbf{x}_i)}
\end{align*}

\subsection{Accuracy of the Approximation}
In this section, we theoretically bound the accuracy of our approximation functions $\hat{Cov}$ and $\hat{Pre}$. The accuracy of the approximation depends on the values of the constants in the definition of $\Gamma$. Note that, by definition $c_4=c_5=0.5$ and $c_3=1-c_1$. Thus we need to tune the parameters $c_1$, $c_h$, $c_l$, and $c_2$. Before we bound $\hat{Cov}$ and $\hat{Pre}$, we would like to make the following observation for the function $\Gamma(x)$ which is defined as $\Gamma(x) = c_1\sigma(c_2x) + c_3(sgn(x)c_4+c_5)$.
\begin{observation}
\label{obs1}
When $c_4=c_5=0.5$ and $c_3 = 1-c_1$, we have:
\begin{itemize}
    \item If $x>0$, $\Gamma(x) = c_1\sigma(c_2x) + c_3$
    \item If $x < 0$, $\Gamma(x) = c_1\sigma(c_2x)$
    \item If $x = 0$, $\Gamma(x) = 0.5$ 
\end{itemize}
\end{observation}

We first begin with bounding the term $h(l,u,x)$. The following Lemma shows $h(l,u,x)$ is a good enough approximation for the indicator function for any poing $x \in \mathbb{R}^d$. 
\begin{lemma}
\label{main_lemma}
Let $c = \frac{c_1}{2}$ and $c_h > 1-c$. If $c < \frac{1}{4D}$, we have $\forall x_i$:
\begin{itemize}
    \item If $ l_j < x_{ij} \le u_j\ \forall j = \{1,2,\ldots, D\}$, we have:
    \begin{align*}
        h(l,u,x) &\le 1 \text{ and }\\
        h(l,u,x) &\ge 1-c
    \end{align*}
    i.e. for all points lying inside the hypercuboid, function $h(\cdot)$ is very close to 1.
    \item If $\exists k,m$ with $k+m \ge 1$, such that $x_{ij} \le l_j$ for $k$ attributes or $x_{ij} > u_j$ for $m$ attributes then:
    \begin{align*}
        h(l,u,x) &\le c \text{ and }\\
        h(l,u,x) &\ge 0
    \end{align*}
    i.e. for all points lying outside the hypercuboid, function $h(\cdot)$ is very close to 0.
\end{itemize}
\end{lemma}
\begin{proof}
The proof considers four cases depending on the number of attributes of a data point that lie between the lower bound and upper bound of the hyperrectangle.\\
\textbf{case 1:} $\forall j\in \{1,2,\ldots, D\}, l_j < x_{ij} \le u_j$. 
\begin{align*}
\allowdisplaybreaks
h(l, u, x_i) &= \Gamma\left(\frac{\sum_{j=1}^{D} \Gamma\left(x_{ij}-l_{j}\right)+\sum_{j=1}^{D} \Gamma\left(u_{j}-x_{ij}+c_{l}\right)}{2 D}-c_{h}\right)\\
&=\Gamma\left( \frac{\sum_{j=1}^{D} (c_1\sigma\left(c_{2}\left(x_{ij}-l_{j}\right)\right)+c_3)}{2D} +\frac{\sum_{j=1}^{D} (c_1\sigma\left(c_{2}\left(u_{j}-x_{ij }+c_l\right)\right)+c_3)}{2 D}-c_{h}\right)\tag*{(From Observation \ref{obs1})}
\end{align*}
Let, $t = c_1\frac{\sum_{j=1}^{D} \sigma\left(c_{2}\left(x_{ij}-l_{j}\right)\right)+ \sum_{j=1}^{D}\sigma\left(c_{2}\left(u_{j}-x_{ij }+c_l\right)\right)}{2D}+c_3-c_h$, then using the fact that $\sigma(x) \ge 0.5 \text{ if } x>0$, we have:
\begin{align*}
    t &\ge \frac{c_1}{2}+c_3 - c_h\\
    &\ge 1-\frac{c_1}{2} - c_h\\
    &> 0 \tag*{(if $c_h + \frac{c_1}{2} < 1$)}
\end{align*}
Thus, if $c_h + \frac{c_1}{2} < 1$, we have $t > 0$. Thus, we get,
$h(l,u,x_i) = \Gamma(t) = c_1\sigma(c_2t) + c_3$ from Observation \ref{obs1}. Since, $t > 0$, $c_2t > 0$ for any $c_2 > 0$, we have,
$h(l,u,x_i) \ge \frac{c_1}{2} + c_3 \ge 1-\frac{c_1}{2}$. Also, $h(l,u,x_i) = c_1\sigma(c_2t) + c_3 \le c_1 + c_3 \le 1$\\

\textbf{Case 2:} Let us assume that $\exists k$ such that $x_{ij} \le l_j$ for $k$ attributes i.e. point lie outside or on the lower bound of hypercuboid for $k \ge 1$ attributes and $\exists m$ such that $x_{ij} > u_j$ for $m \ge 1$ attributes. Out of $k$ attributes, let $k_1$ attributes have $x_{ij} = l_j$ and $k-k_1$ attributes $x_{ij} < l_j$. Then, we have:
\begin{itemize}
    \item For all $k_1$ attributes: $\Gamma(x_{ij} - l_j) = 0.5$ 
    \item For $k-k_1$ attributes:$\Gamma(x_{ij} - l_j) = c_1\sigma(c_2(x_{ij}-l_j)) \le c_1$
    \item For $D-k$ attributes: $\Gamma(x_{ij} - l_j) = c_1\sigma(c_2(x_{ij}-l_j)) + c_3 \le 1$
    \item For all $m$ attributes: $\Gamma(u_j - x_{ij} + c_l) = c_1\sigma(c_2(u_j-x_{ij} + c_l)) \le c_1$
     \item For $D-m$ attributes: $\Gamma(u_j - x_{ij} + c_l) = c_1\sigma(c_2(u_j-x_{ij} + c_l)) + c_3 \le 1$
\end{itemize}
We get,
\begin{align*}
    &h(l,u,x_i)\\ 
    &= \Gamma\left(\frac{0.5k_1+\sum_{j=1}^{k-k_1}c_1\sigma(c_2(x_{ij}-l_j))}{2D} + \frac{\sum_{j=k+1}^D((c_1\sigma(c_2(x_{ij}-l_j)) + c_3)}{2D}\right.\\
    &\left. + \frac{\sum_{j=1}^m c_1\sigma(c_2(u_j-x_{ij} + c_l)}{2D} + \frac{\sum_{j=m+1}^D c_1\sigma(c_2(u_j-x_{ij} + c_l) + c_3}{2D} - c_h\right)
\end{align*}
Let, $h(l,u,x_i) = \Gamma(t)$ i.e. consider the entire term in $\Gamma$ expression to be $t$ then:
\begin{align*}
    t \le &\frac{0.5k_1 + (k-k_1)0.5c_1 + (D-k)(c_1 + c_3)}{2D} +\frac{0.5mc_1 + (D-m)(c_1+c_3)}{2D} - c_h\\
    t \le &\frac{0.5k_1(1-c_1) + 0.5kc_1 + 0.5mc_1+2D-k-m}{2D}-c_h \tag*{($1 \le k+m \le 2D$, $k_1 \le D$, and $0 < c_1 < 1$)}\\
    \le & \frac{0.5D(1-c_1) + 0.5c_1D}{2D}+\frac{2D-1}{2D}-c_h\\
    \le&\frac{1}{4D} + \frac{2D-1}{2D} - c_h \le \frac{4D-1}{4D} - c_h
\end{align*}
Thus, when $c_h > \frac{4D-1}{4D}$, then we get $t < 0$. In this case, we have: $h(l,u,x_i) = \Gamma(t) = c_1\sigma(c_2t) \le \frac{c_1}{2}$. From Case 1, we have $\frac{c_1}{2} < 1-c_h$. Substituting $c_h > \frac{4D-1}{4D}$, we get, $\frac{c_1}{2} < \frac{1}{4D}$. Thus, if any of the attribute of the example lies outside the boundary, we get $h(l,u,x_i) \le \frac{1}{4D}$ and if all the attributes lie inside the boundary, we get $h(l,u,x_i) \ge \frac{4D-1}{4D}$
\end{proof}

Thus, choosing $c_1$ and $c_h$ according to the lemma ensures that $h$ is a good approximation to the indicator function $\bm{1}(\textbf{x} \in S(\textbf{l},\textbf{u}))$. Further, we can arrive at the following important result that bounds the difference between $Cov$ and the corresponding approximation $\hat{Cov}$.
 \begin{theorem}
 If $c_1<\frac{1}{2D}$ and $c_h>\frac{4D-1}{4D}$, then
 \begin{equation}\left(\frac{4D-1}{4D}\right) Cov \le \hat{Cov} \le \frac{1}{4D} + \left(\frac{4D-1}{4D}\right)Cov \nonumber
 \end{equation}
\end{theorem}

\begin{proof}
Let the actual coverage from the hypercuboid $(l,u)$ be $\frac{k}{N}$ i.e. $\sum_{i=1}^N\mathbb{I}(x_i \in S(l,u)) = k$. Then:
\begin{align*}
\allowdisplaybreaks
    \hat{Cov} &= \frac{1}{N}\sum_{i=1}^N h(l,u,x_i)\\
    &=\frac{1}{N}\sum_{x_i \in S(l,u)} h(l,u,x_i)+\frac{1}{N}\sum_{x_i \notin S(l,u)} h(l,u,x_i)\\
    &\ge \frac{1}{N}k(1-c)\tag*{(From Lemma \ref{main_lemma})}\\
    &\ge Cov\left(\frac{4D-1}{4D}\right) \tag*{($c <\frac{1}{4D}$)}
\end{align*}

Also,
\begin{align*}
\allowdisplaybreaks
    \hat{Cov} &=\frac{1}{N}\sum_{x_i \in S(l,u)} h(l,u,x_i)+\frac{1}{N}\sum_{x_i \notin S(l,u)} h(l,u,x_i)\\
    &\le \frac{k}{N} + \frac{N-k}{N}c\tag*{From Lemma \ref{main_lemma}}\\
    &\le c + Cov\left(1-c\right)\\
    &\le \frac{1}{4D} + Cov\left(\frac{4D-1}{4D}\right)\tag*{($c <\frac{1}{4D}$)}
\end{align*}
\end{proof}
The above result is interesting not only because it bounds the approximate coverage in terms of true coverage but it also suggests that as the features (dimension) increases, approximate coverage becomes closer to the true coverage. We also verify this from our experiments in Table \ref{tab:approximation}.

We also have additional result for the bounds on the approximate precision.
\begin{theorem}
$\hat{Pre} \le Pre\left(1 + \frac{1}{Cov}\left(\frac{4D}{4D-1}\right)\right)$. Thus, when algorithm returns a hypercuboid with  $\hat{Pre} \ge P$ then $Pre \ge \frac{1}{\left(1 + \frac{1}{Cov}\left(\frac{4D}{4D-1}\right)\right)}P$
\end{theorem}
\begin{proof}
Let, $k$ points be inside the hyper-cuboid, out of $k$ points, $q$ points satisfy $f(x_i) = f(x_q)$ and $m$ points satisfy $f(x_i) = f(x_q)$ in total.
\begin{align*}
    \hat{Pre} &= \displaystyle \frac{\sum_{\substack{x_i = x_q\\ x_i \in S(l,u)}}h(l,u,x_i) + \sum_{\substack{x_i = x_q\\ x_i \notin S(l,u)}}h(l,u,x_i)}{\sum_{x_i \in S(l,u)}h(l,u,x_i) + \sum_{x_i \notin S(l,u)}h(l,u,x_i)}\\
    &\le \frac{q+(m-q)c}{(1-c)k}\\
    &\le Pre + \frac{m}{k}\left(\frac{c}{1-c}\right)\\
    &\le Pre + \frac{qN}{k^2}\left(\frac{4D}{4D-1}\right) \tag*{$\left(\frac{M}{N} \le \frac{q}{k} \text{ and }c < \frac{1}{4D}\right)$}\\
    &\le Pre + \frac{Pre}{Cov}\left(\frac{4D}{4D-1}\right)
\end{align*}
\end{proof}

In summary, when the dimension of the dataset increases, $c_1\approx 0$ and $c_h\approx 1$, and the difference between analytical coverage and the corresponding approximation tends to 0.
The accuracy of the approximation was further studied experimentally using a synthetic 1-D dataset. Figure \ref{fig:comp_plots} presents the plots of the analytical coverage, precision, and the corresponding estimates for a synthetic 1D dataset with [0.3, 0.7] representing the positive class. The query point is 0.5. It is evident from the figures the actual values and their approximations match quite well. Even though the theoretical bounds depend on the dimensionality of the data, we conclude from the experiments that the values $c_1=0.4, c_2=15, c_3=0.6, c_4=0.5, c_5=0.5, c_l=0.02$, and $c_h = 0.8$ work well for a wide variety of datasets and do not have to be tuned for a new dataset. We use these values for all the experiments performed in the paper.

\begin{figure}[t]
    \centering
    \includegraphics[width=\textwidth]{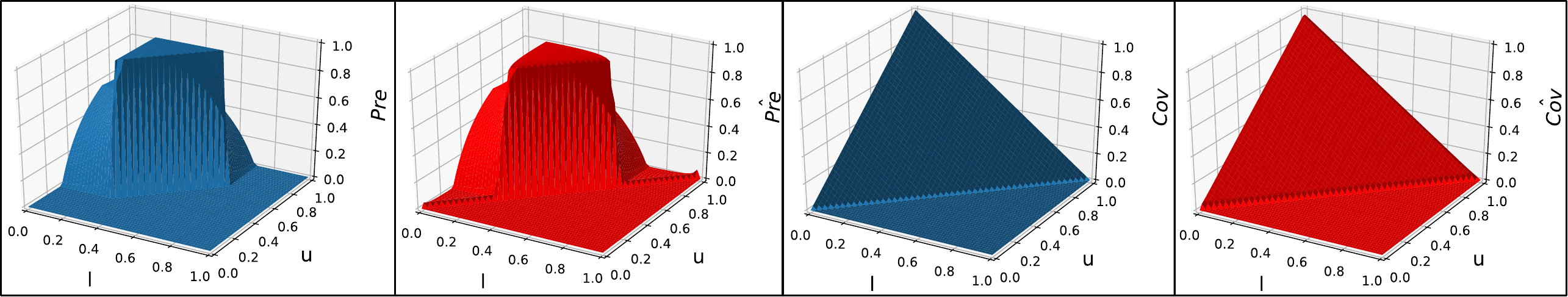}
    \caption{Comparison of $Cov$, $\hat{Cov}$, $Pre$, $\hat{Pre}$ for a 1D case for $\textbf{x}'_q = [0.5]$ where the interval [0.3, 0.7] has positive class label.}
    \label{fig:comp_plots}
\end{figure}

\section{Optimization to Estimate the Explanation}
Our next objective is to formulate the optimization problem in the MAIRE framework for estimating the explanation. In the simplest case, we want to find the optimal values for the parameters $\textbf{l}$ and $\textbf{u}$ that maximize coverage while maintaining a minimum precision $P$. The user sets this lower bound on precision. It is assumed that any value of $Pre(\textbf{l}, \textbf{u})$ above $P$ is acceptable. If the current value of precision is greater than the threshold, we would only like to maximize the coverage. If the analytical precision becomes less than the specified threshold, then in addition to maximizing coverage, the MAIRE framework also maximizes the precision $\hat{Pre}$. This component is weighted by a constant factor $\lambda_1$ to signify the importance of increasing the precision at the cost of reducing the coverage. If $\lambda_1$ is high, $\hat{Pre}$ will increase whenever the precision is less than the threshold. Thus the overall objective function $\mathcal{L}(\textbf{l}, \textbf{u})$ is defined as:

\begin{align*}
\allowdisplaybreaks
 \mathcal{L}(\textbf{l}, \textbf{u}) = \hat{Cov}(\textbf{l}, \textbf{u}) + \lambda_1 \hat{Pre}(\textbf{l}, \textbf{u})(1 + \text{sgn}(P - Pre(\textbf{l}, \textbf{u})))
\end{align*}
Note that the analytical precision value $Pre(\textbf{l} ,\textbf{u})$ is only required to activate the approximation term. 
The objective function $\mathcal{L}(\textbf{l}, \textbf{u})$ is maximized subject to two constraints. First, the lower and upper bound vectors $\textbf{l}$ and $\textbf{u}$ need to be in $[0, 1]^D$. As $\Gamma(z)$, the approximator to the indicator function, never truly achieves a zero gradient, if the explanation is unbounded, then the optimization procedure might never converge as the explanation could keep expanding in all directions indefinitely. This constraint is implemented by clipping the values of $\textbf{l}$ and $\textbf{u}$ at 0 and 1 respectively after every iteration. The second constraint is that the explanation finally generated must contain the query instance: $l_j \leq x'_{qj} \leq u_j, \forall j=1,\ldots, D$.  The optimizer focuses on these constraints, only when they are not satisfied. When the constraints are satisfied, the optimizer only maximizes the coverage. This is achieved by using the $\text{ReLU}$ function on the difference $l_j-x'_{qj}$ and $x'_{qj}-u_j$. These constraints are added to the final optimization function with a weighting constant $\lambda_2$ (can be viewed as the Lagrange multiplier used for constrained optimization) that signifies the penalty on the objective when the constraint is not satisfied. Thus, the final objective function in the MAIRE framework that is maximized with respect to the parameters $\textbf{l}$ and $\textbf{u}$ is defined as follows
\begin{equation}
    \argmax_{\textbf{l}, \textbf{u}} \mathcal{L}(\textbf{l}, \textbf{u})- \lambda_2 \left(\sum_{j = 1}^D ReLU(l_j - x'_{qj}) + \sum_{j = 1}^D ReLU(x'_{qj} - u_j)\right)
\end{equation}
Adam optimizer \cite{adam} is used for this non-linear and non-convex optimization with default parameter values for obtaining the results. On an average across the datasets, with a learning rate of 0.01, the Adam optimizer required around 2500 iterations to converge. 

\subsection{Greedy Attribute Elimination for Human Interpretability}
The explanations created might still be too large (containing non-trivial bounds for many dimensions) for a human to understand. We reduce the size of the generated explanations using a greedy elimination procedure to improve human interpretability. An attribute whose removal results in a maximum increase in the coverage while retaining precision above the user-defined threshold is eliminated. If no such attribute exists, then the attribute whose removal reduces the precision by the minimum extent is excluded from the explanation. Attributes are removed greedily at least for $D-K$ times, where $K$ is the maximum number of attributes that can be part of an explanation as set by the user. Note that once we get the the hypercuboid $S(l,u)$, the greedy selection of a single feature will take $O(D)$ time because computing coverage and precision for a given hypercuboid with one feature removal will take constant time to compute.

\subsection{Local to Global Explanations}
To gain a broader understanding of how the classification model works on the entire dataset, we would need to generate multiple explanations for a comprehensive set of instances. This is an infeasible task due to the significant computational complexity. Instead of creating a global explanation by combining local explanations of randomly selected instances, we identify an optimal set of local explanations that can approximate the global behavior of the classification model. 

The process of creating a global explanation is started by considering a moderately sized subset of the training set (chosen randomly). Local explanations are generated using the MAIRE framework for all the instances in this set. A subset of these explanations is selected greedily, such that every new local explanation added to the global explanation leads to the maximum increase in the overall coverage of the global explanation. We call this procedure Maximum Symmetric Difference (MSD Select) as the local explanation that results in the maximum symmetric difference with the current estimate of the global explanation is added to the global explanation.

The global explanation can be viewed as a new rule-based classifier $f'(\textbf{x})$. Given a test data point, several local explanations that are part of $f'(\textbf{x})$ can be applied to predict the class label. We 
propose to use the majority class label among the applicable explanations for generating the class label.


\subsection{Extension to Discrete Attributes}
The MAIRE framework is directly applicable on ordered discrete attributes. The final explanation is a set of consecutive discrete values. The generated explanation is slightly modified for ordered discrete attributes by changing $l_i$ to the smallest discrete value that is greater than or equal to $l_i$ and changing $u_i$ to the largest discrete value that is lesser than or equal to $u_i$. This modification does not affect coverage or precision and improves readability. In the case of a categorical attribute (unordered), finding intervals is not meaningful. We instead convert all categorical attributes to their equivalent one-hot encoding. The transformed boolean representation is treated as ordered discrete attributes. If an explanation contains both the values of a boolean attribute, the corresponding attribute is dropped from the explanation.  If only the value one is selected, then the value of the unordered attribute in $\textbf{x}'_q$ is included in the explanation. Due to the enforcement of the second constraint, selection of only 0 is not possible as $\textbf{x}'_q$ has the value 1 for the corresponding boolean attribute.

\section{Experiments and Results}
Code for the experiments mentioned is available at \url{https://github.com/anonymousID2242/code-submission}. The MAIRE framework is tested on a wide variety of synthetic and real-world datasets. The instances for all synthetic datasets are sampled from the interval $[0, 1]$. For these datasets, a simple shape was chosen for positive class ($f(\textbf{x}) = 1$) region. Everywhere else, $f(\textbf{x})$ is 0. Using simple shapes allows for easy visualization of the explanations generated by the model. 
\subsection{Synthetic Datasets}
\begin{figure*}
\centering
\subfigure[$f(\textbf{x}'_q) = 1$]{
\includegraphics[width=0.45\textwidth]{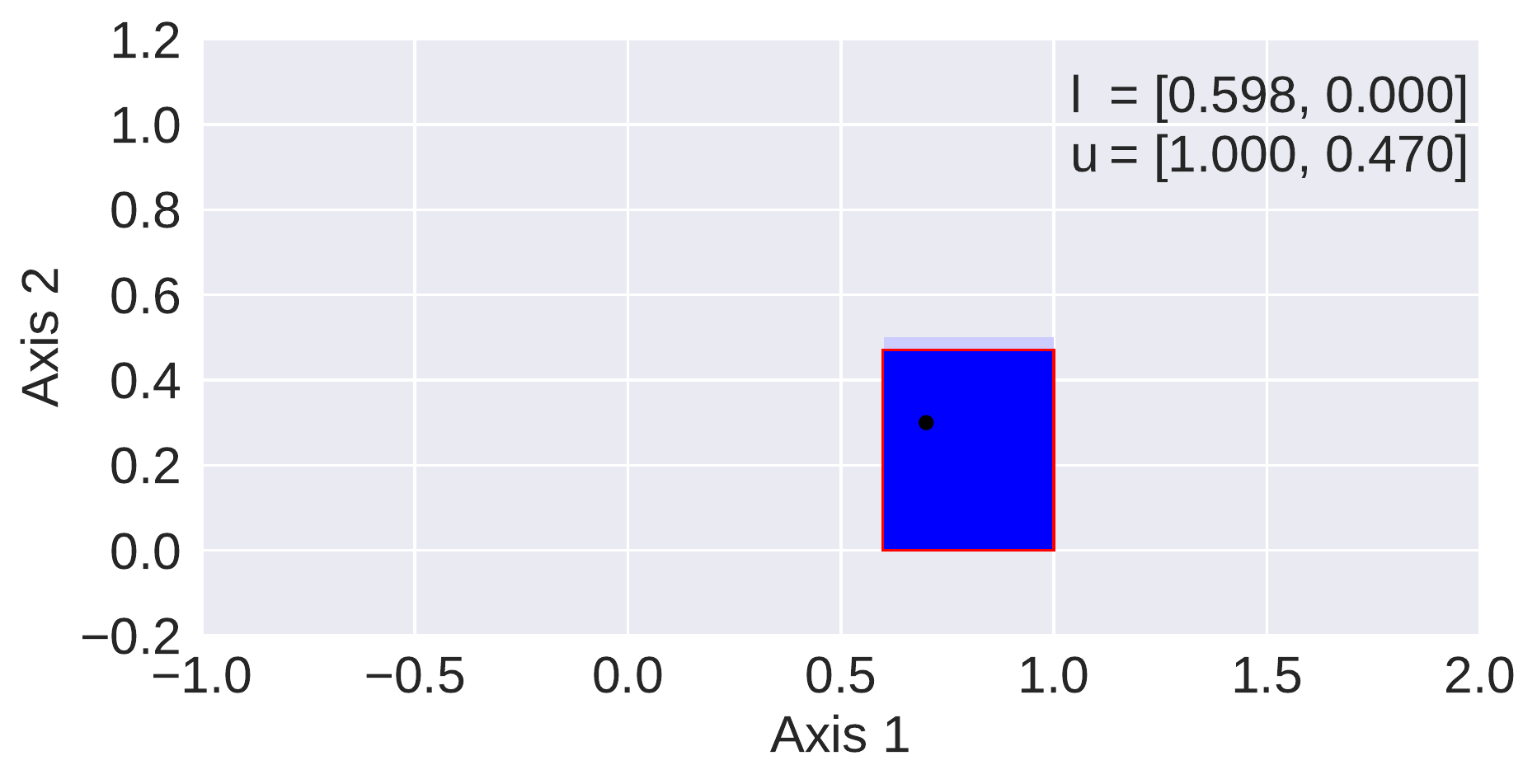}
\label{fig:syn_0a}
}%
\subfigure[$f(\textbf{x}'_q) = 0$]{
\includegraphics[width=0.45\textwidth]{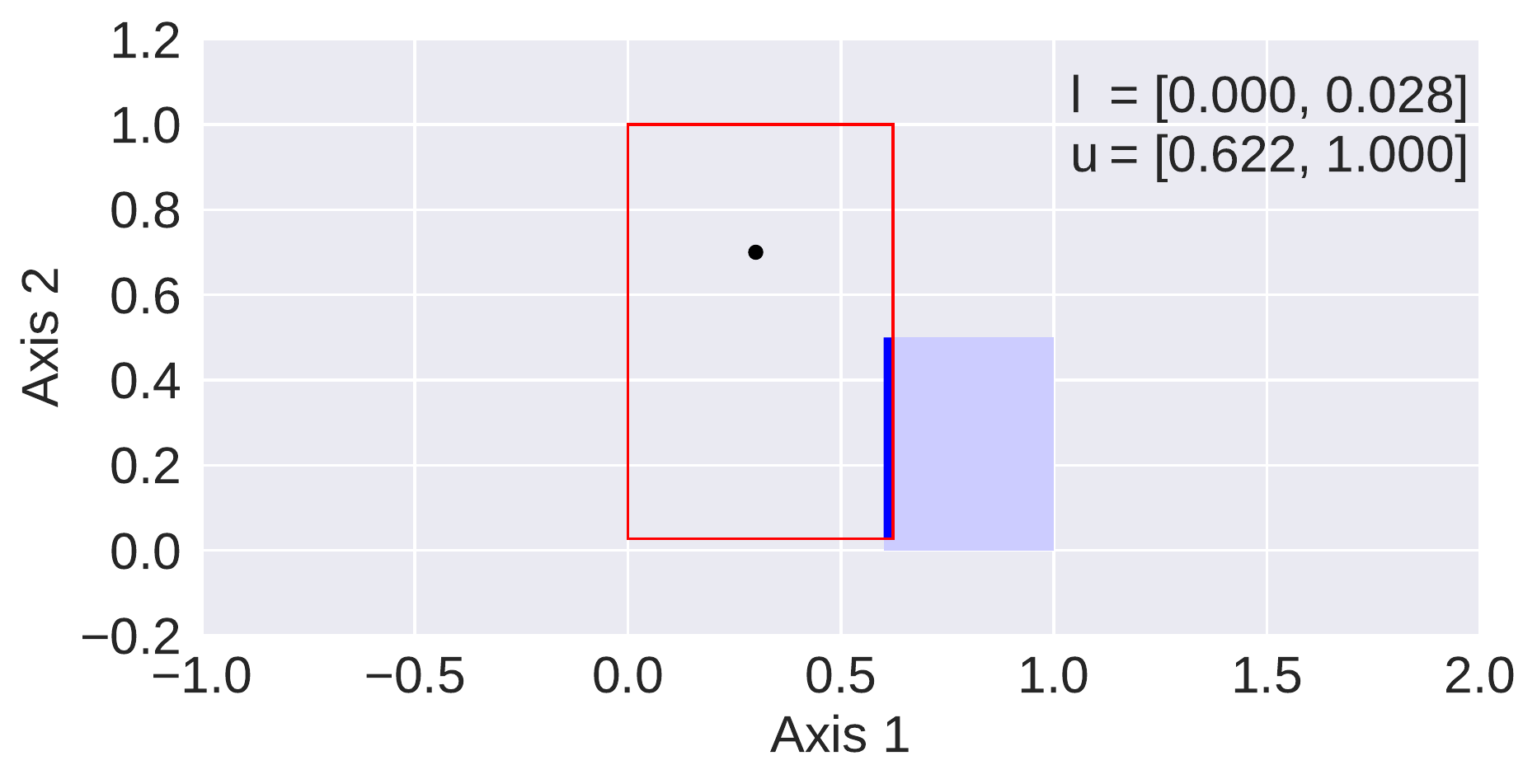}
\label{fig:syn_0b}
}
\subfigure[P = 0.80]{
\includegraphics[width=0.45\textwidth]{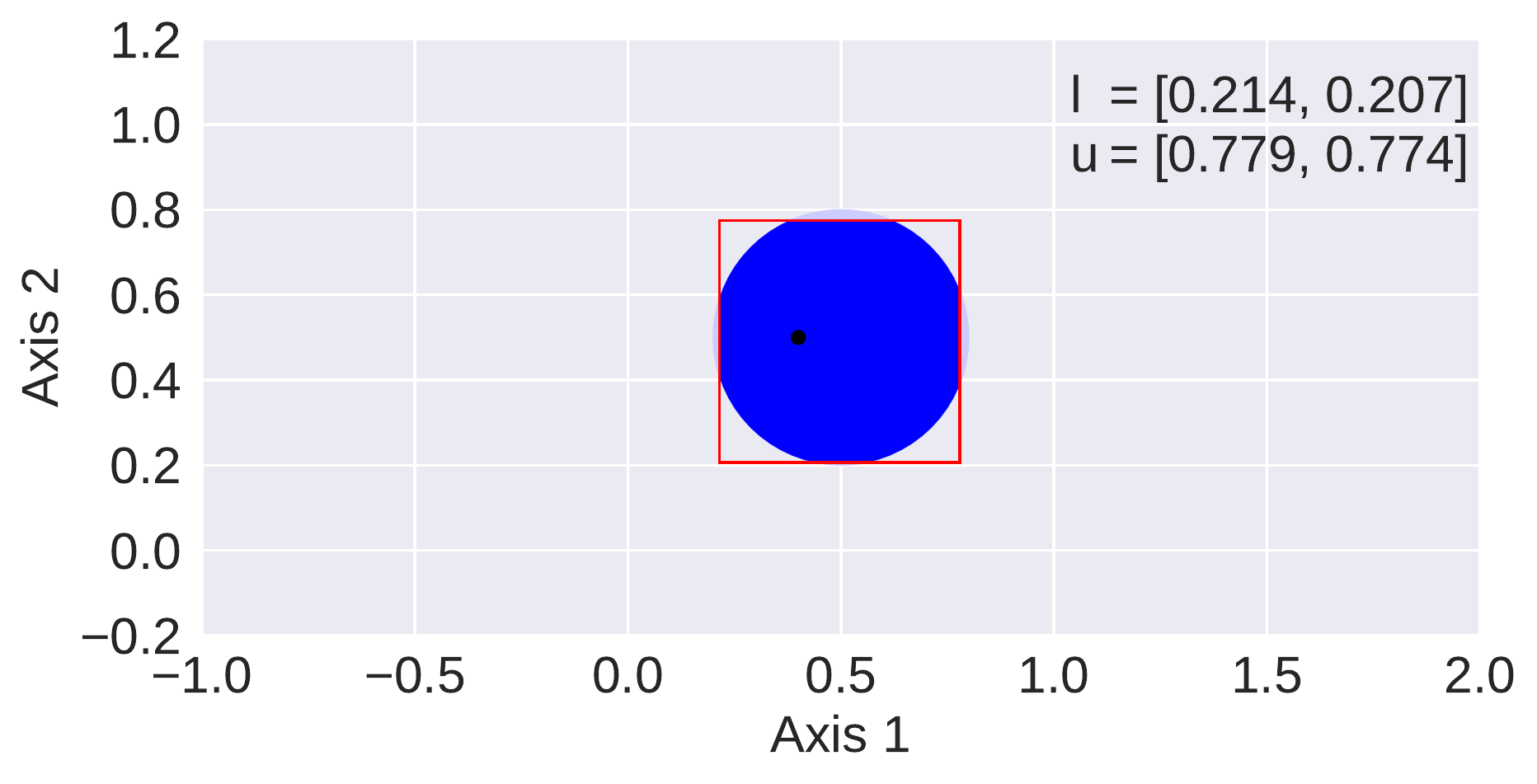}
\label{fig:syn_1a}
}
\subfigure[P = 0.95]{
\includegraphics[width=0.45\textwidth]{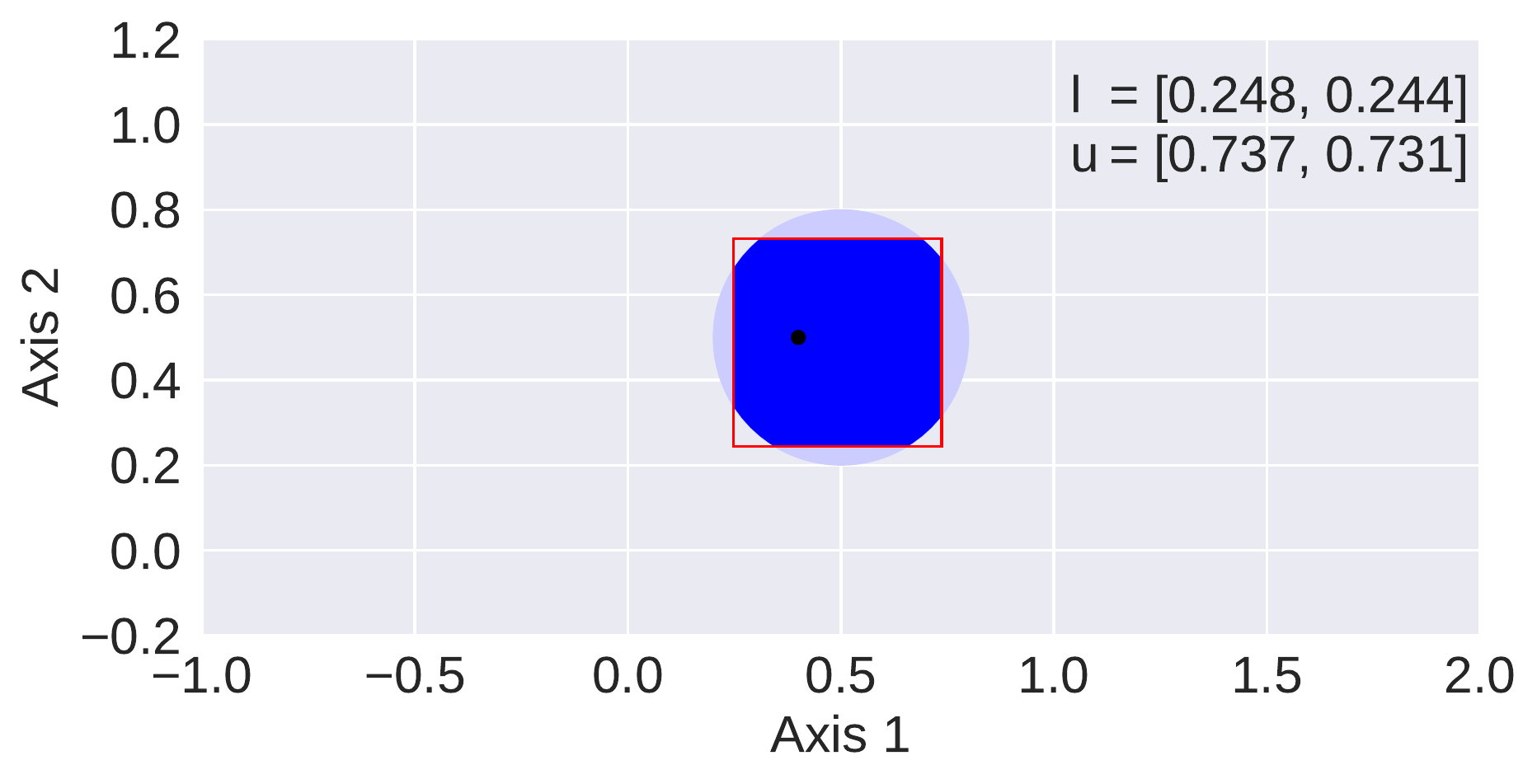}
\label{fig:syn_1b}
}
\subfigure[Second Constraint Off ($\lambda_2 = 0$)]{
\includegraphics[width=0.45\textwidth]{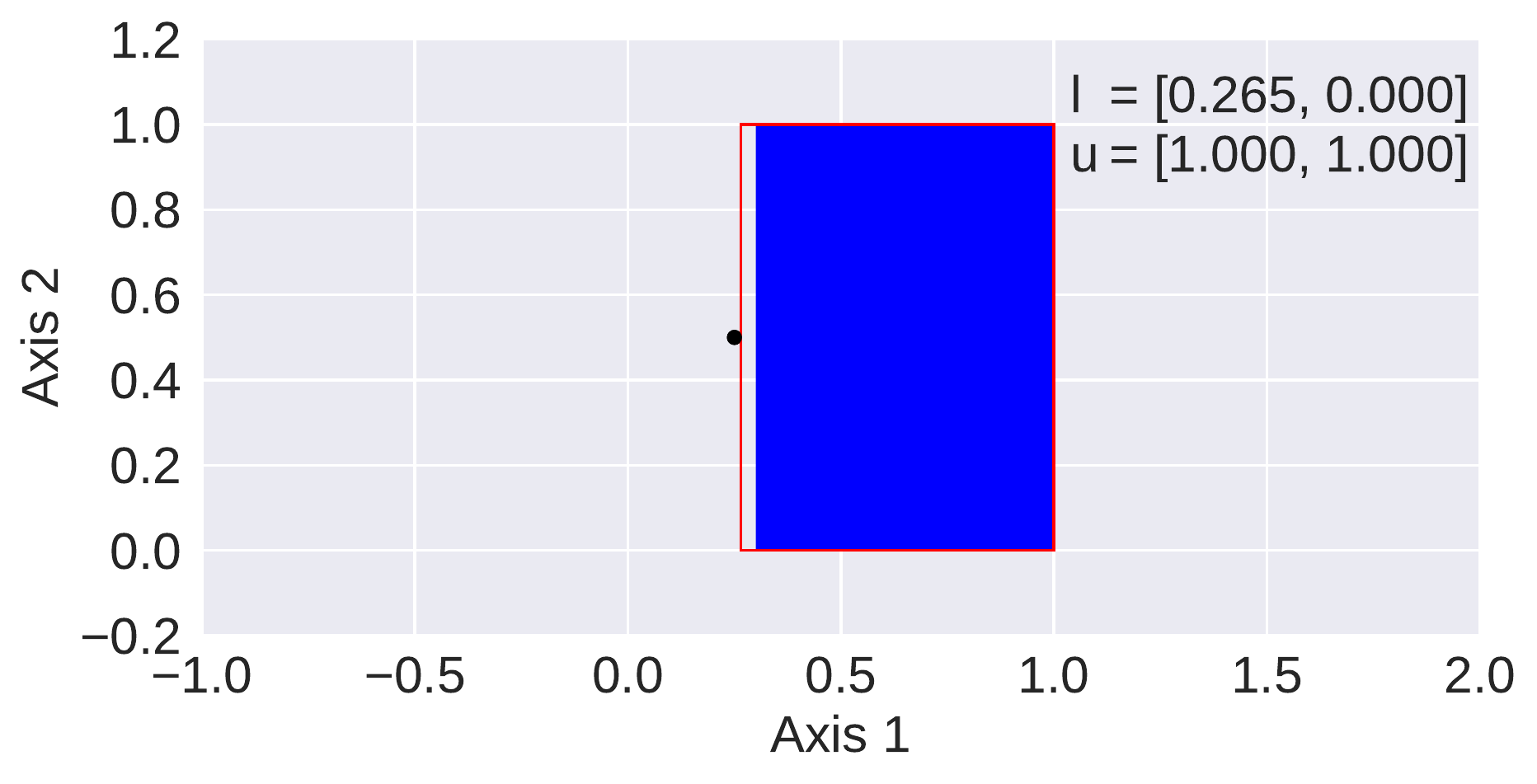}
\label{fig:syn_2a}
}
\subfigure[Second Constraint On ($\lambda_2 = 5$)]{
\includegraphics[width=0.45\textwidth]{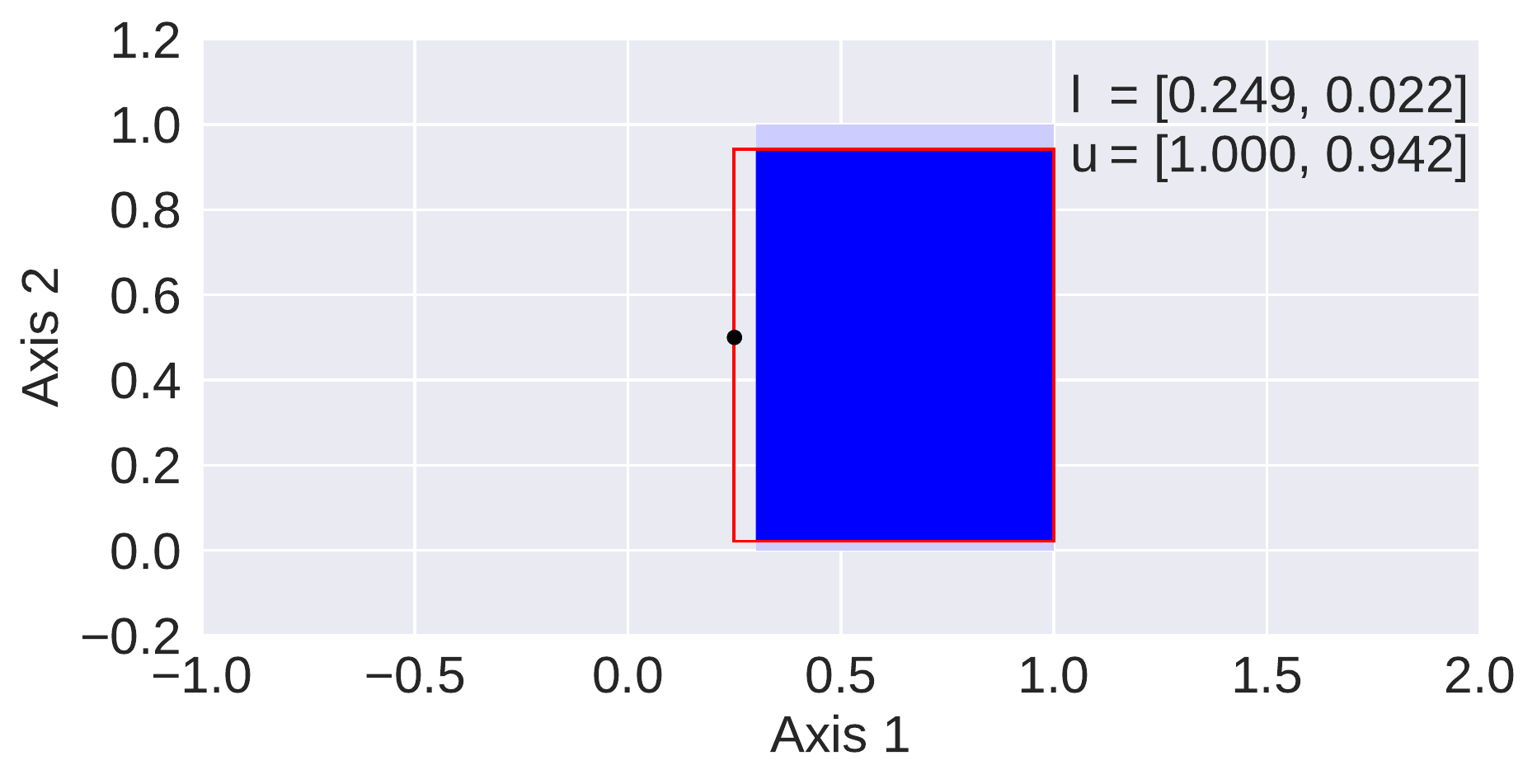}
\label{fig:syn_2b}
}
\subfigure[$f(\textbf{x}'_q) = 1$]{
\includegraphics[width=0.45\textwidth]{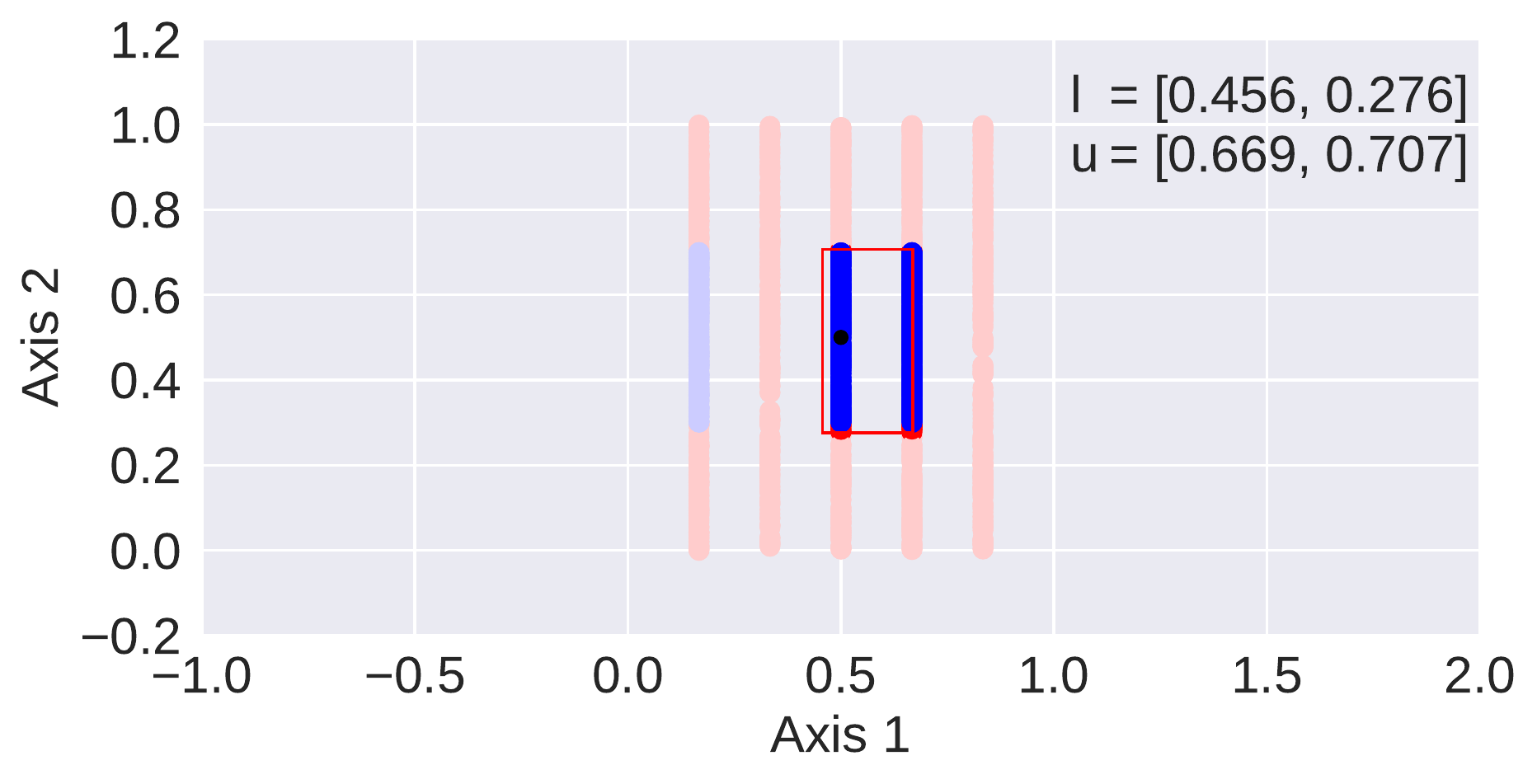}
\label{fig:syn_3a}
}%
\subfigure[$f(\textbf{x}'_q) = 0$]{
\includegraphics[width=0.45\textwidth]{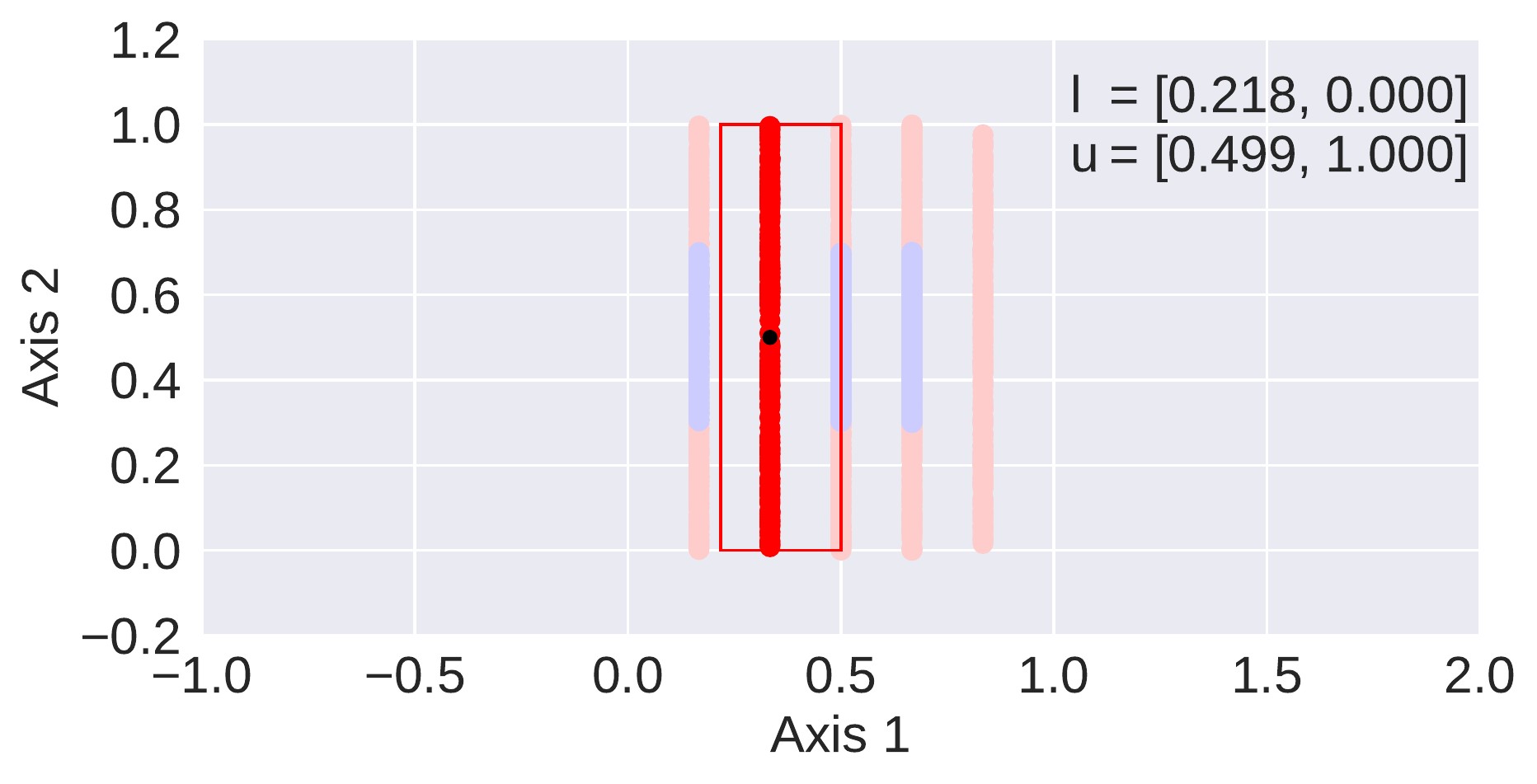}
\label{fig:syn_3b}
}
\caption{[Best viewed in color]MAIRE Explanations for Synthetic Datasets (a) and (b) for Rectangular Decision Boundaries, (c) and (d) for Circular Decision Boundaries, (e) and (f) Effect of the Second Constraint (with $\textbf{x}'_q = [0.250, 0.500]$), (g) and (h) for Synthetic Datasets with Discrete Attributes.}
\label{fig:syn_plots}
\end{figure*}

Figure \ref{fig:syn_plots} illustrates the explanations generated by the MAIRE framework on various synthetic datasets. In all of these figures, the blue regions represent $f(\textbf{x}) = 1$, the red rectangle marks the final explanation generated by the framework and the black point refers to the query point, i.e., $\textbf{x}'_q$. The lighter colors are used for marking regions that are not included in the explanations. In Figures \ref{fig:syn_plots} (a)-(f), the non-blue regions represent $f(\textbf{x}) = 0$. In Figures \ref{fig:syn_plots} (g) and (h), the red regions represent $f(\textbf{x}) = 0$ and non-blue regions are not included in the instance space, i.e. the instance space is discretized. The attribute along axis 1 is discretized to take 5 values - $\{\frac{1}{6}, \frac{2}{6}, \frac{3}{6}, \frac{4}{6}, \frac{5}{6}\}$.

Figures \ref{fig:syn_plots}(a) and (b) represents the MAIRE explanations on a rectangular decision boundary with different query points. We observe that in Figure \ref{fig:syn_plots}(a), because $\textbf{x}'_q$ belongs to $f(\textbf{x}) = 1$ region, the explanation completely covers the rectangle as this is the largest region that includes $\textbf{x}'_q$ and has a high precision. Similarly, in Figure \ref{fig:syn_plots}(b), $\textbf{x}'_q$ belongs to $f(\textbf{x}) = 0$ region. So, the explanation generated is correspondingly the largest rectangle that includes $\textbf{x}'_q$ such that most of the region has $f(\textbf{x}) = 0$, thus having a high precision. An interesting observation here is that the framework has two potential options - horizontally cover the whole range or vertically cover the whole range. We want to point out that for the MAIRE framework, these correspond to two local maxima. Vertical cover has a larger area and so, is the global maxima. We observe that the MAIRE explanation corresponds to the vertical cover. However, it could have very well chosen the other local maxima, i.e., the horizontal cover instead. This is an artifact of any gradient-based optimization routine. 

Figures \ref{fig:syn_plots}(c) and (d) represent the MAIRE explanations on a circular decision boundary with different values of the precision threshold $P$. In Figure \ref{fig:syn_plots}(c), as $P$ was 0.80, we observe that the final explanation almost completely circumscribes the circle. On the other hand, in Figure \ref{fig:syn_plots}(d), as $P$ was 0.95, the explanation generated is smaller in size as this size has lesser percentage of points with $f(\textbf{x}) = 0$.

Figures \ref{fig:syn_plots}(e) and (f) compare the explanations generated by the MAIRE framework when the second constraint (i.e., the explanation must contain $\textbf{x}'_q$) is active and inactive. In this set of experiments, there are two $f(\textbf{x}) = 1$ regions. One is marked in blue (the blue rectangle). Other than that, $f(\textbf{x}'_q)$ is also 1.

In the Figure \ref{fig:syn_plots}(e), the constraint was inactive (with $\lambda_2 = 0$). We observe that the final explanation does not contain $\textbf{x}'_q$. This is simply because the framework maximizes the precision by minimizing the thin $f(\textbf{x}) = 0$ strip. In the Figure \ref{fig:syn_plots}(f), the constraint was active (with $\lambda_2 = 5$). We observe that the final explanation contains $\textbf{x}'_q$. Here, the entire vertical range was not covered because that would have led to a precision lower than the threshold $P$.

Figures \ref{fig:syn_plots}(g) and (h) represents the MAIRE explanations for a synthetic dataset where one attribute has an ordered discrete domain and the other has a continuous domain. In the Figure \ref{fig:syn_plots}(g), $f(\textbf{x}'_q) = 1$ and so, the corresponding blue regions from the two adjacent strips were selected. While in the Figure \ref{fig:syn_plots}(h), $f(\textbf{x}'_q) = 0$ and so, the corresponding red strip was selected as the explanation.


\subsection{Tabular Datasets}
We conducted experiments to study the quality of the approximations to coverage and precision using the tabular datasets. Explanations for 100 randomly sampled data points for each of the datasets were computed. The true coverage and precision were determined for each explanation as well as the values for the corresponding approximations. The mean squared error between the true and approximate values averaged over 100 data points for the three datasets is presented in Table \ref{tab:approximation}. It can be noticed that difference in the true values and the corresponding approximations is not significant. Further this difference reduces as the number of attributes increases supporting our theoretical analysis. The German credit dataset has the highest number of attributes (20), followed by Adult (14), and Abalone (8) data sets.

The MAIRE framework is evaluated on three tabular datasets - Adult, Abalone, and German credit datasets. A three-layer neural network (containing 150, 100, and 50 nodes in each layer with ReLU activation) serves as the black-box model (though any classifier can serve the purpose). The datasets are divided into train and test splits according to the ratio of 3:1. The neural network is trained for 100 epochs. The test accuracy of the black box model for the Adult, Abalone, and German credit datasets is 81.52\%, 87.76\%, and 79.28\%, respectively. A sample of the explanations generated by MAIRE for the three datasets is presented in Table \ref{tab:my-table}.

We compare the quality of the global explanations extracted from MAIRE against other model-agnostic rule-based explanation methods capable of composing global explanations, namely; LIME and Anchors. LIME and Anchor are applicable only on discrete datasets. Hence, for a fair comparison, we have used the same discretized version of the dataset across all the models, including MAIRE. The precision threshold is set at 0.95 for all the datasets. We compare the sub-modular pick (SP) versions of LIME and Anchor against the MSD Select of MAIRE. Coverage over unseen test instances in the global explanation is used as the metric for comparison.  Figure \ref{fig:msd} a-c presents the results averaged over five trials on the three tabular datasets. MSD-MAIRE consistently performs better than SP-LIME and SP-Anchor, achieving the maximum coverage using a lesser number of explanations. Thus MSD-MAIRE has higher coverage at the same precision threshold. It is also observed that SP-LIME performs better than SP-Anchors on the German-credit dataset.

We further conduct experiments on the original non-discretized version of the tabular datasets only using MAIRE. We compare the global explanation created by MSD-MAIRE against randomly selected local explanations - RP-MAIRE. The results on the adult dataset are presented in Figure \ref{fig:msd}d. Similar trends were observed for the Abalone and German-credits datasets. 
\begin{table}[h]
\centering
\begin{tabular}{|l|l|l|l|}
\hline
              & Adult  & Abalone & German credit \\ \hline
MSD Coverage  & 0.0015 & 0.0004  & 8.552e-05     \\ \hline
MSD Precision & 0.3217 & 0.1265  & 0.0985        \\ \hline
\end{tabular}
\caption{Mean Square difference between $Cov$ and $\hat{Cov}$, $Pre$ and $\hat{Pre}$ for adult, abalone and German credit datasets averaged over 100 data points.}
\label{tab:approximation}
\end{table}

\begin{figure*}[!t]
\centering
\subfigure[]{
\includegraphics[width=0.3\textwidth]{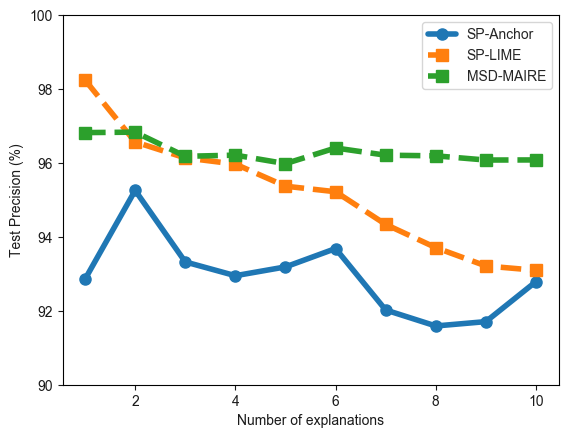}
}
\subfigure[]{
\includegraphics[width=0.3\textwidth]{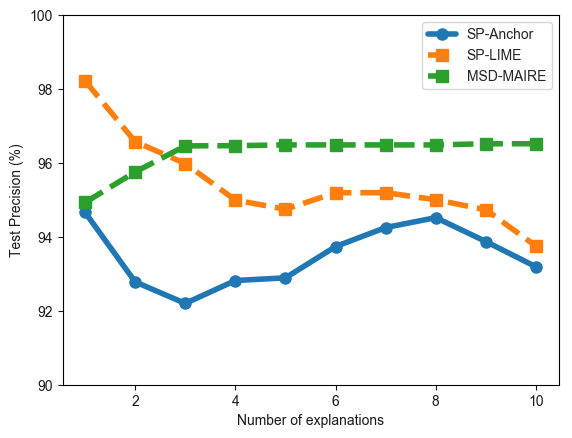}
}
\subfigure[]{
\includegraphics[width=0.3\textwidth]{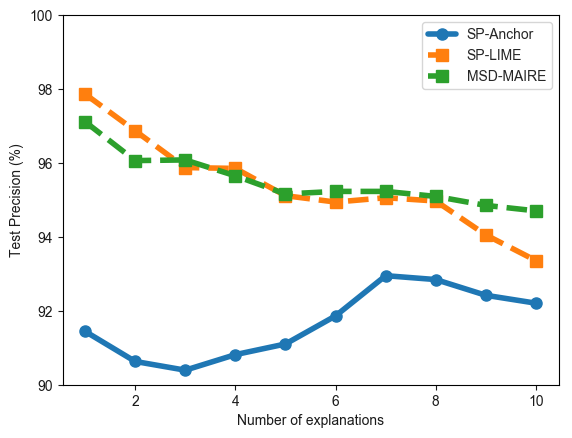}
}
\caption{[Best viewed in color] Change in test precision as a function of number of local explanations included in the global explanation Test Coverage for (a) Adult (b) Abalone (c) German-Credit datasets for SP-LIME, SP-Anchors and MSD-MAIRE.}
\label{fig:precision}
\end{figure*}

Figures \ref{fig:precision} (a-c) compare the change in precision as the local explanations are incrementally added to the global explanation for the three tabular datasets. It is observed that the proposed framework results in a minimal reduction in precision consistently across the three datasets. The observation is in line with the mechanism the MAIRE framework employs to create a global explanation ensuring a minimum reduction in precision. LIME shows the maximum decrease in precision.

Figures \ref{fig:msd}(a-c) compares the performance of the MAIRE framework for both the discretized and non-discretized versions of the tabular datasets. We observe that the coverage of the global explanation for MSD-MAIRE for both versions of the datasets is comparable for Adult and Abalone datasets. However, we notice a significant improvement in the performance of MAIRE on the discretized version of the German-Credit dataset. Further investigation is required to understand this anomaly.


\begin{figure*}[!t]
\centering
\subfigure[]{
\includegraphics[width=0.4\textwidth]{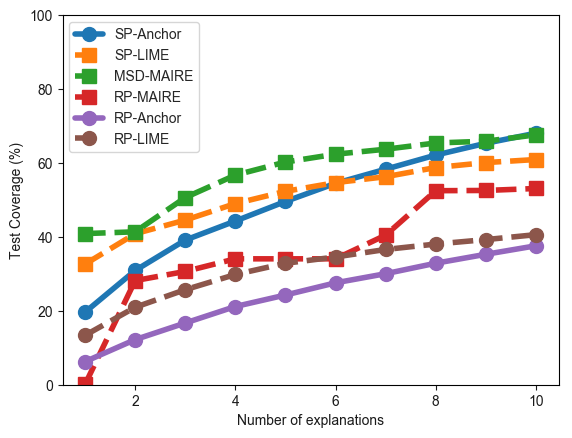}
}
\subfigure[]{
\includegraphics[width=0.4\textwidth]{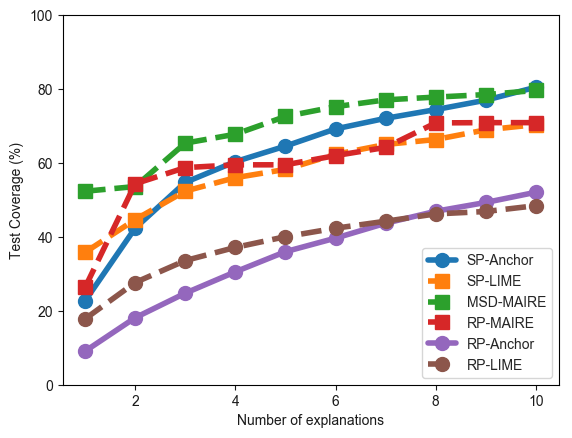}
}
\\
\subfigure[]{
\includegraphics[width=0.4\textwidth]{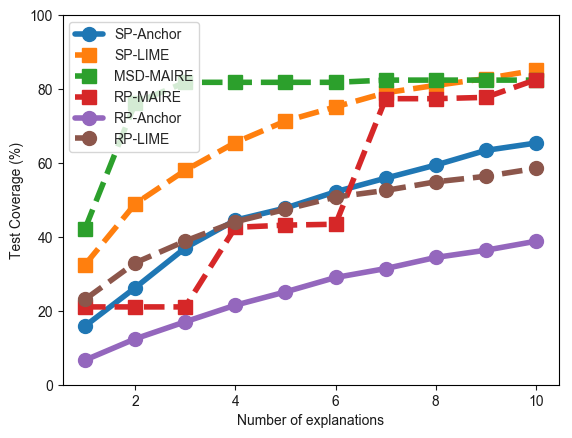}
}
\subfigure[]{
\includegraphics[width=0.4\textwidth]{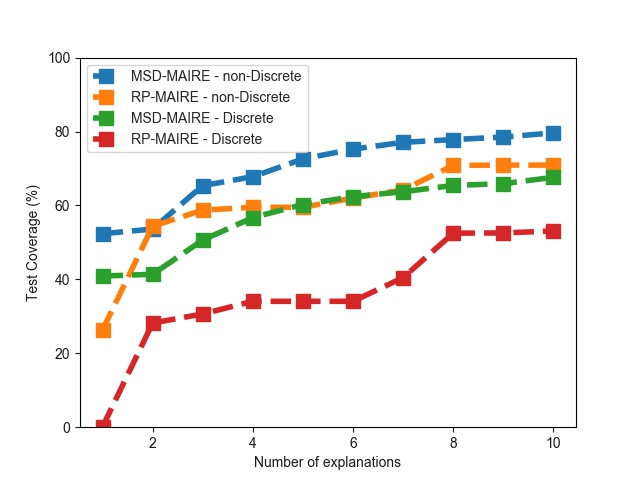}
}
\caption{[Best viewed in color]Change in test coverage as a function of number of local explanations included in the global explanation Test Coverage for (a) Adult (b) Abalone (c) German-Credit Data sets (d) Comparison of discretized vs non-discretized version of the datasets for RP-MAIRE and MSD-MAIRE}
\label{fig:msd}
\end{figure*}

\subsection{Text Datasets}
The MAIRE framework is evaluated on two text datasets- IMDB movie reviews and a reduced 20-Newsgroups dataset (containing the data belonging to the four classes - medicine, graphics, Christian, and atheism). We illustrate the model-agnostic capability of the MAIRE framework by training a decision forest classifier for the IMDB dataset and a deep learning classifier for the Newsgroup dataset. The datasets are divided into train and test splits in the ratio of 4:1. A bag of words representation was used to characterize the reviews and documents. 

In the case of IMDB movie reviews, we considered a random forest with 500 trees as our black-box model to be explained.   In the case of 20-Newsgroup dataset, we have only considered output labels `medicine,'`graphics,' `Christian,' and `atheism' as it is not feasible to present 20 labels to a human subject. We use a four-layer network consisting of two hidden layers with 512 nodes each, having ReLU activation, and dropout probability set to 0.3 among layers, and softmax activation at the output layer as the base classifier.  The model is trained for 30 epochs using Adam optimizer. The test classification accuracy on the two datasets is 87.3\% and 81.2\%, respectively.  

We use ten data points (three medicine, three atheism, two graphics, two Christian) and generated explanations for each review using 5 different approaches mentioned in the paper. For generating the MAIRE explanation, the review was converted into a bag of words vector, and the sample points for computing $Cov, \hat{Cov}, Pre, \text{ and } \hat{Pre}$ were taken by randomly flipping bits in the bag of words. The words are ranked based on the effect they have on the classification using Greedy Attribute Elimination. Tables    
\ref{tab:textexplanations} and \ref{tab:textmisclassifiedexplanations} present explanations generated by various explanatory models for both correct and incorrect classifications by the base classifier.

We conduct human subject experiments on the explanations for ten random test instances for each of the datasets to compare MAIRE against other model-agnostic approaches, namely; LIME and Anchors and feature ranking approaches, namely; L2X, and SHAP. We employ the experimental protocol of Chen et al., \cite{chen2018learning} for computing human accuracy. We assume that the explanations, in terms of the keywords (maximum of 10), convey sufficient information about the sentiment or class label of the document. We ask human subjects to infer the sentiment or the class label of the text when provided with only the explanations. The explanations from the different models and the various instances of a dataset are randomized. The final label for each document is averaged over the results of 25 human annotators. We measure the accuracy of the label predicted by the human annotator against the output of the model. The subjects are also allowed to label an explanation ``can not infer"  if the explanation is not sufficiently informative. We use the \textit{Human Accuracy} metric for comparing the different approaches and treat the instances labeled as ``can not infer" as misclassified instances.

The results are reported in Table \ref{tab:textha}. The human judgment given only ten words aligns best with the model prediction when the words are chosen from L2X and MAIRE for the IMDB and Newsgroup datasets, respectively. While on the binary classification dataset (IMDB), L2X is better than MAIRE by around 3\%, on the more challenging 4-way classification dataset (Newsgroup) MAIRE leads over L2X by 6\%. Overall the result indicates the competitiveness of MAIRE against other feature ranking approaches. It is also evident that MAIRE has significantly higher human accuracy over the other model-agnostic approaches LIME and Anchors. Table ~\ref{tab:sample_explanations} shows the results of the various models for two examples.

\begin{table}[!t]
\small
\centering
    \begin{tabular}{|c|c|c|c|c|c|}
    \hline
    Method & LIME & SHAP & Anchor & L2X & MAIRE \\
    \hline
    IMDB & 0.66 & 0.56 & 0.62 & 0.70 & 0.67 \\
    \hline
    Newsgroup & 0.66 & 0.64 & 0.70 & 0.69 & 0.75 \\
    \hline
    \end{tabular}
    \caption{Human accuracy of various model agnostic approaches on IMDB and Newsgroup datasets.}
    \label{tab:textha}
\end{table}

\begin{table*}[]
\small
    \centering
    \begin{tabular}{|p{6cm}|p{1.5cm}|p{1.5cm}|p{1.5cm}|p{1.5cm}|p{1.5cm}|}
    \hline
    Review/Document & LIME & SHAP & Anchors & L2X & MAIRE\\
    \hline
    I have to say that this miniseries was the best interpretation of the beloved novel ``Jane Eyre". Both Dalton and Clarke are very believable as Rochester and Jane. I've seen other versions, but none compare to this one. The best one for me. I could never imagine anyone else playing these characters ever again. The last time I saw this one was in 1984 when I was only 13. At that time, I was a bookworm and I had just read Charlotte Bronte's novel. I was completely enchanted by this miniseries and I remember not missing any of the episodes. I'd like to see it again because it's so good. :-) & best, completely, believable, 13, say, just, imagine, good, read, remember & beloved, none, interpretation, good, novel, missing, remember, best, read, imagine & believable remember, best, novel & imagine, interpretation, best good, novel, just, remember, read, characters, believable & enchanted, best, interpretation, remember, good, believable, novel, imagine, beloved, completely \\
    \hline
    In article 47974@sdcc12.ucsd.edu| wsun@jeeves.ucsd.edu (Fiberman) writes: Is erythromycin effective in treating pneumonia?
    It depends on the cause of the pneumonia. For treating bacterial pneumonia in young otherwise-healthy non-smokers, erythromycin is usually considered the antibiotic of choice, since it covers the two most-common pathogens: strep pneumoniae and mycoplasma pneumoniae. & cause, treating, edu, common, effective, healthy, usually, antibiotic, bacterial, non & healthy, writes, common, cause, effective, young, pneumoniae, choice, treating, cover & pneumonia, healthy, antibiotic & cause, treating, antibiotic, edu, young, covers, bacterial, pathogens, choice, considered, & common, bacterial, covers, young, pathogens, healthy, usually, smokers, cause, pneumoniae \\
    \hline
    \end{tabular}
    \caption{Sample Explanations for documents in the IMDB and Newsgroup Dataset.}
    \label{tab:sample_explanations}
\end{table*}

\subsection{Image Datasets}
We use the MAIRE framework to explain the classification results of the VGG16 model \cite{vgg}. For explaining the model output, the image is segmented into superpixels and each superpixel is treated as a Boolean attribute. $\textbf{x}'_q$ is taken to be a vector of 1 indicating the presence of all superpixels in the image. Sample points for calculating $Cov$, $\hat{Cov}$, $Pre$ and $\hat{Pre}$ are computed by flipping the bits of $\textbf{x}'_q$ randomly (i.e. randomly removing some superpixels). In the final explanation, the superpixels that covered both the values \{0, 1\} of the corresponding Boolean attributes are removed as these superpixels do not affect the decision of the classifier.
Figures \ref{fig:beagle} (a-c) show the explanation generated by the MAIRE framework and the heat map of the explanation (generated by ordering the superpixels chosen in the local explanation using Greedy Attribute Elimination) for a sample image (beagle). The VGG model has high confidence in its prediction for this image. 
\begin{figure*}[!t]
\centering
\subfigure[]{
\includegraphics[width=0.3\textwidth]{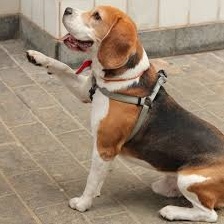}
}
\subfigure[]{
\includegraphics[width=0.3\textwidth]{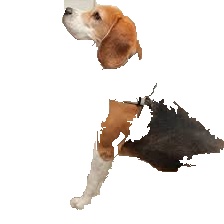}
}
\subfigure[]{
\includegraphics[width=0.3\textwidth]{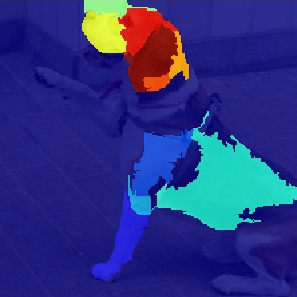}
}
\subfigure[]{
\includegraphics[width=0.3\textwidth]{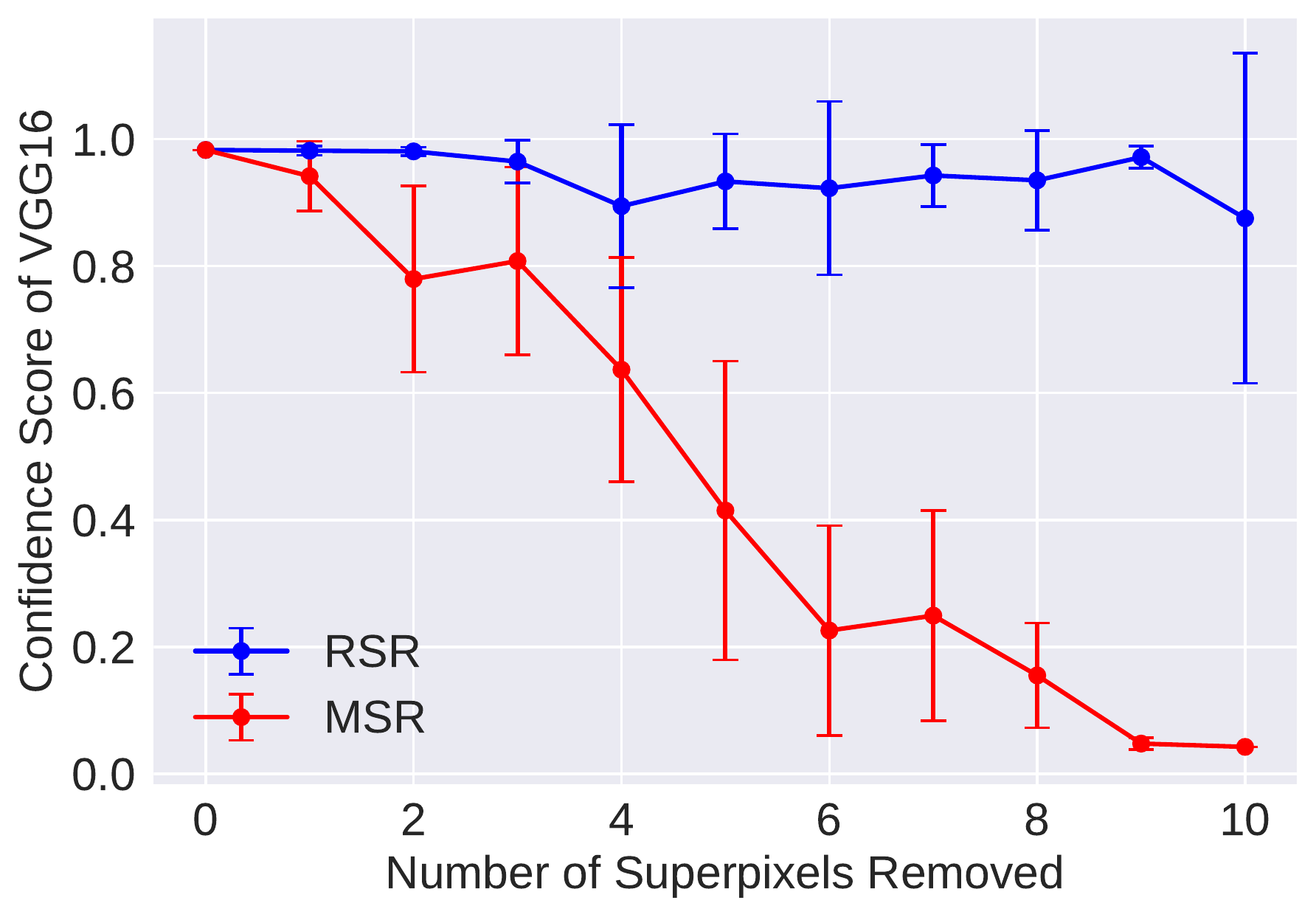}
}
\subfigure[]{
\includegraphics[width=0.3\textwidth]{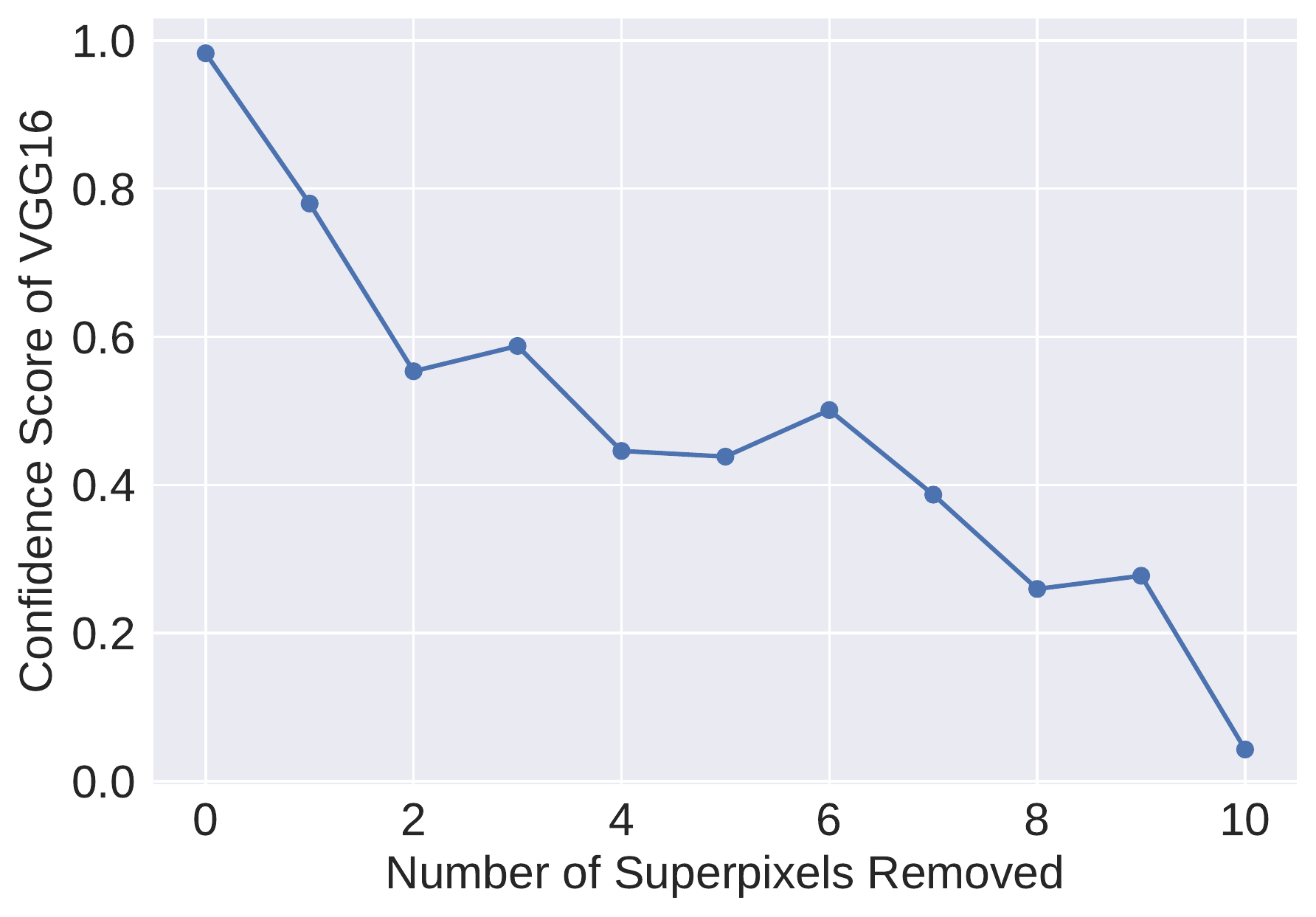}
}
\caption{[Best viewed in color]Results on the beagle image (a) Original Image (b) Explanation Generated (c) Heat Map (d) Confidence Score as more number of Superpixels are Removed (RSR = Random Superpixels Removed, MSR = MAIRE Superpixels Removed) (e) Confidence Score as more number of Superpixels are Removed (the removal order is from most important to least as given by Greedy Attribute Elimination)}
\label{fig:beagle}
\end{figure*}

\begin{figure*}[!t]
\centering
\subfigure[]{
\includegraphics[width=0.3\textwidth]{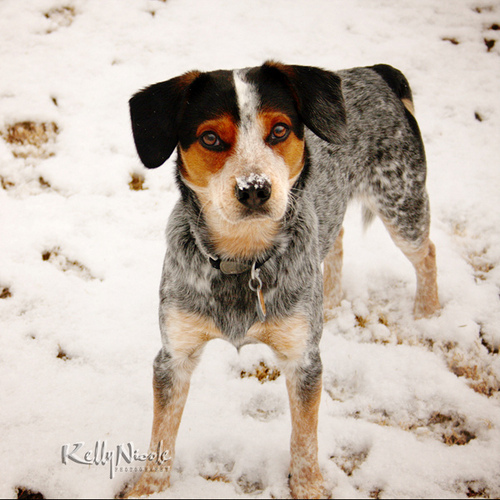}
}
\subfigure[]{
\includegraphics[width=0.3\textwidth]{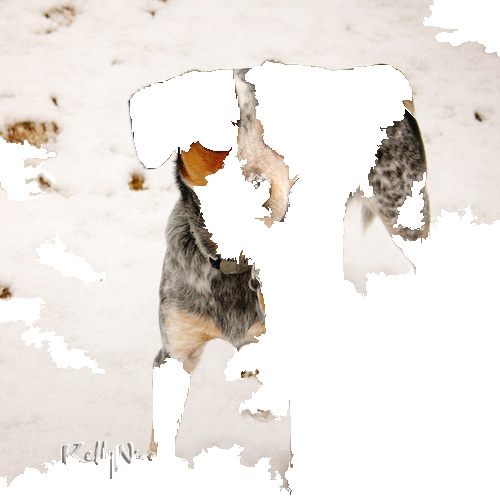}
}
\subfigure[]{
\includegraphics[width=0.3\textwidth]{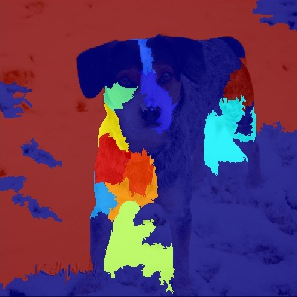}
}
\subfigure[]{
\includegraphics[width=0.3\textwidth]{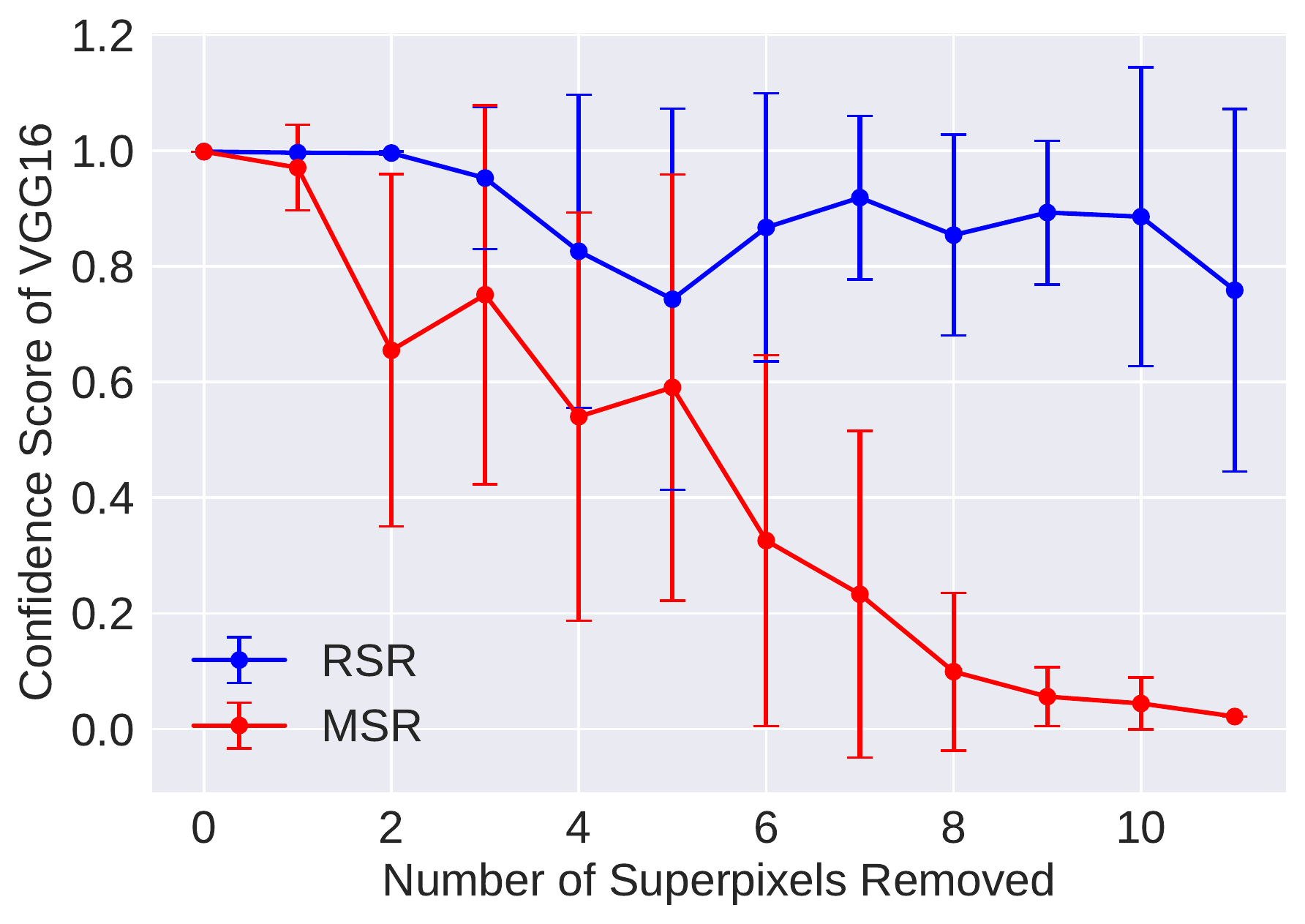}
}
\subfigure[]{
\includegraphics[width=0.3\textwidth]{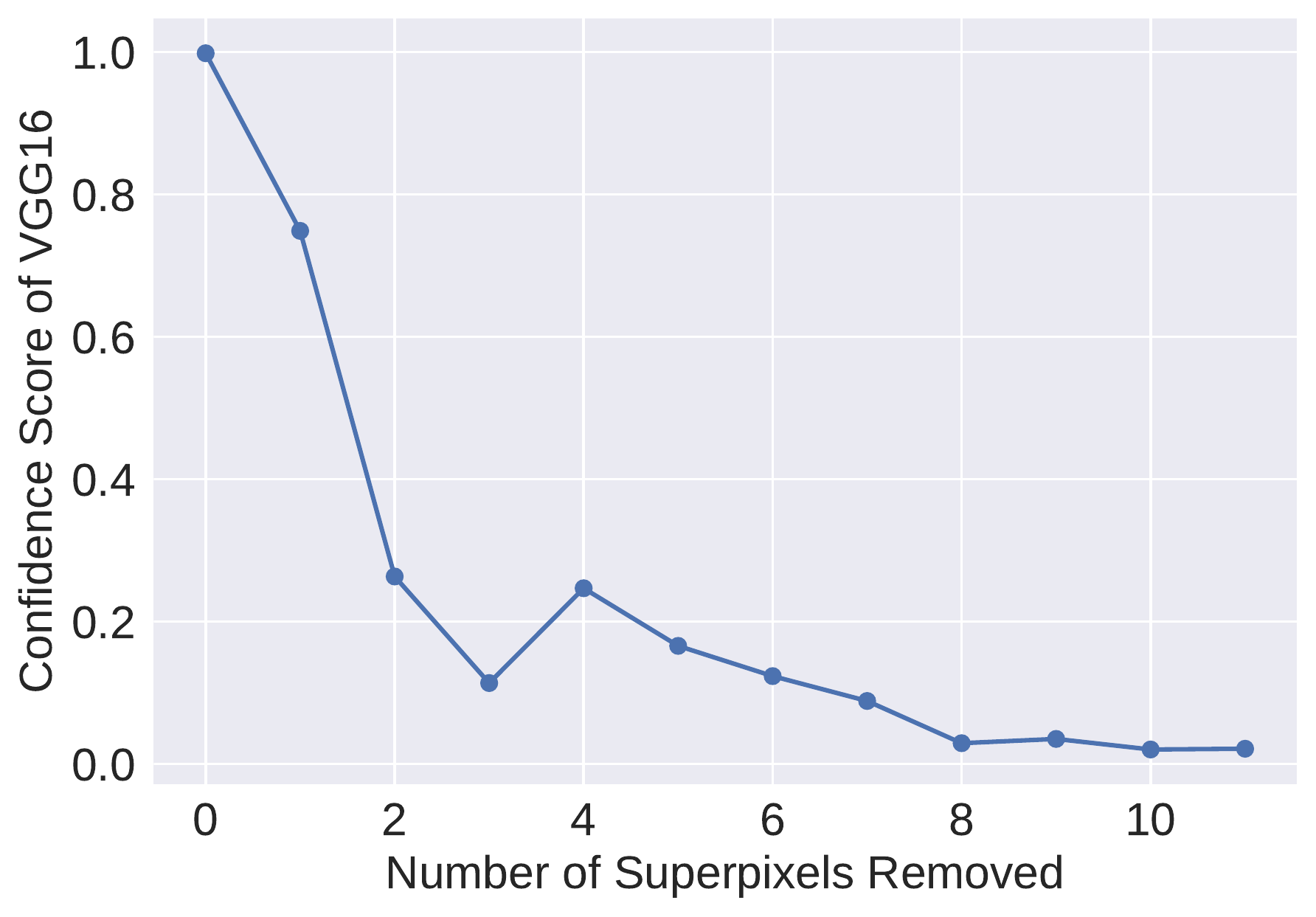}
}
\caption{[Best viewed in color] Results on the bluetick image (a) Original Image (b) Explanation Generated (c)Heat Map (d) Confidence Score as more number of Superpixels are Removed (RSR = Random Superpixels Removed, MSR = MAIRE Superpixels Removed) (e) Confidence Score as more number of Superpixels are Removed (the removal order is from most important to least as given by Greedy Attribute Elimination)}
\label{fig:bluetick}
\end{figure*}

We first validate the performance of the MAIRE framework by measuring the classifier confidence when random superpixels are removed (RSR) and when the superpixels picked by the MAIRE framework for the explanation are removed (MSR) from the original image. The results of this experiment are presented in Figure \ref{fig:beagle}(d). It is observed that the decrease in the classifier confidence on the removal of superpixels picked by the MAIRE framework is significantly larger than randomly selecting a superpixel. This illustrates that the MAIRE framework does indeed select the superpixels that have a big impact on the classifier.

In the second experiment, we only pick the superpixels selected by the greedy algorithm in the MAIRE framework. We iteratively remove the selected superpixels in the decreasing order of importance as estimated by the greedy algorithm, while also computing the classifier confidence. Our hypothesis is that if the greedy algorithm does indeed pick only important superpixels, then we would expect a sharp drop in the classifier confidence when the initial set of superpixels are removed from the image. Figure \ref{fig:beagle}(e) presents the results for the beagle image. We observe that by removing the top 4 superpixels selected by the MAIRE framework, the classifier confidence drops to less than 0.5. 

The Figure \ref{fig:bluetick} shows the explanation generated by the MAIRE framework and heat map of the explanation (generated by ordering the superpixels chosen in the local explanation using Greedy Attribute Elimination) for the bluetick image. The VGG model has high confidence in its prediction for this image as well. The decrease in the classifier confidence (Figure \ref{fig:bluetick}(d)) with the removal of superpixels picked by the MAIRE framework is more significant than randomly selecting a superpixel. This also illustrates that the MAIRE framework does indeed select the superpixels having a significant impact on the classifier. We also observe that by removing the top 2 superpixels selected by the MAIRE framework, the classifier confidence drops to less than 0.5 for the bluetick image. It is interesting to note that the images in Figures \ref{fig:bluetick}(b) and \ref{fig:bluetick}(c) show that the MAIRE framework selected superpixels mostly from the background in the bluetick image. Surprisingly, the VGG16 model classified the image, containing only the superpixels selected by the MAIRE framework for the bluetick image, correctly with the confidence of 0.953. Further, when we remove the superpixel containing the background snow, the VGG16 classifier confidence drops to 0.007. This indicates that the VGG16 network is focusing on perhaps incorrect regions of the image. The MAIRE framework is effective at detecting such wrong correlations learned by the machine learning model.

\subsection{Comparison Against Decision Trees}
\begin{figure}[h]
\centering
\subfigure[]{
 \includegraphics[width=0.3\textwidth]{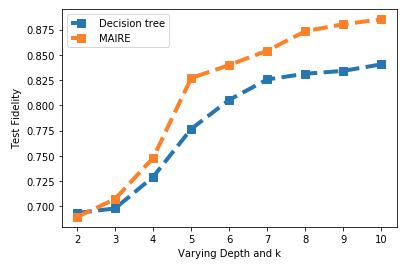}
}
\subfigure[]{
\includegraphics[width=0.3\textwidth]{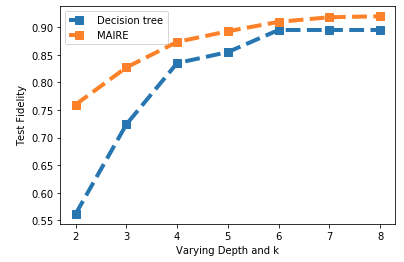}
}
\subfigure[]{
\includegraphics[width=0.3\textwidth]{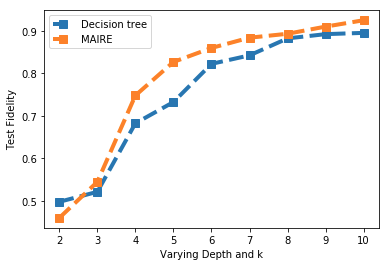}
}
\\
\subfigure[]{
\includegraphics[width=0.3\textwidth]{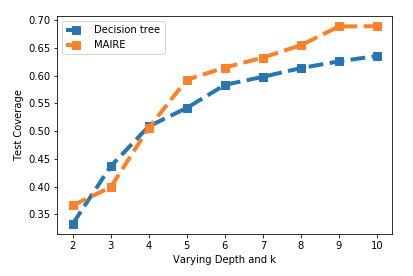}
}
\subfigure[]{
\includegraphics[width=0.3\textwidth]{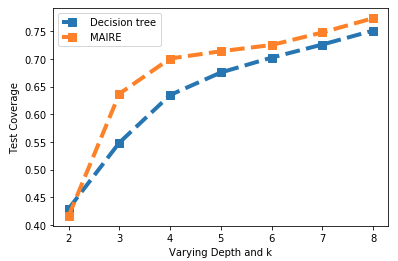}
}
\subfigure[]{
\includegraphics[width=0.3\textwidth]{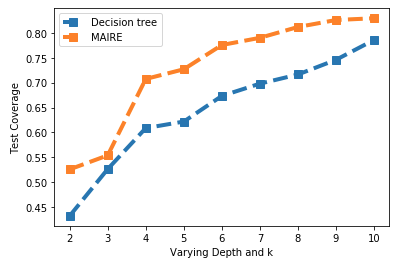}
}
\caption{(a), (b), (c) Precision of MAIRE and Decision Tree for Adult, Abalone and German Credit datasets respectively, (d), (e), (f)  Coverage of MAIRE and Decision Tree for for Adult, Abalone and German Credit datasets respectively}
\label{fig:rule_based}  
\end{figure}
  
The MAIRE framework generates rules that are similar to explanations created by a rule-learner such as a decision tree. We compare the precision (also referred to as fidelity in rule-learning literature) and interpretability of the two methods. Interpretability is measured in terms of the length of the number of antecedents in a rule generated by the explanatory model. A large number of conditions in the \textit{if-then-else} rule makes it difficult for a human to interpret the explanation.

We use all the three tabular datasets and a neural network as a black-box model for the comparison. A global explanatory model is extracted using 200 training data points from both MAIRE and a decision tree. We learn different global explanations for varying values of the parameter $K$ that refers to the number of conditions in the MAIRE explanation as well as the depth of the decision tree. The precision of the global explanation for each value of $K$ is measured on an unseen test set. A global explanation may not predict the black-box model's output for some data points due to limited coverage. The union of such data points not explained by either of the global explanation models is left out when computing precision for both the models. The result of this experiment is presented in Figure \ref{fig:rule_based}. It is observed that MAIRE has higher precision than decision trees for almost all values of $K$. Further, for small $K$, the coverage of both the models is similar. This indicates that MAIRE is able to generate better explanations in terms of both precision and interpretability than a decision tree.

\section{Summary}
In this paper, we propose a novel model-agnostic interpretable rule extraction (MAIRE) framework for explaining the decisions of black-box classifiers. The framework quantifies the goodness of the explanations using coverage and precision. We propose novel differentiable approximations to these measures that are then optimized using the gradient-based optimizer. The flexible framework can be applied to any classifier for a wide variety of datasets. We test the framework on multiple datasets (tabular, text, and image) and show that the generated explanations are competitive to state-of-the-art approaches.










 \begin{table*}[]
\small
    \begin{tabular}{|p{6cm}|p{1.5cm}|p{1.5cm}|p{1.5cm}|p{1.5cm}|p{1.5cm}|}
    \hline
    Review/Document & LIME & SHAP & Anchors & L2X & MAIRE\\
    \hline
    \textbf{model prediction : negative,  True label : negative} \newline 
Encouraged by the positive comments about this film on here I was looking forward to watching this film. Bad mistake. I've seen 950+ films and this is truly one of the worst of them - it's awful in almost every way: editing, pacing, storyline, 'acting,' soundtrack (the film's only song - a lame country tune - is played no less than four times). The film looks cheap and nasty and is boring in the extreme. Rarely have I been so happy to see the end credits of a film. The only thing that prevents me giving this a 1-score is Harvey Keitel - while this is far from his best performance he at least seems to be making a bit of an effort. One for Keitel obsessives only.  & worst, Bad, awful, lame, boring, best, cheap, acting, thing, effort   & mistake, best, lame, pacing, extreme, credits, obsessives, far, cheap, happy   & bad, storyline, nasty, boring   & credits, worst, comments, awful, cheap, nasty, mistake, extreme, lame, effort   & Worst, awful, lame, boring, mistake, less, bad, cheap, obsessives, extreme  \\ 
    \hline
    \textbf{model prediction : graphics,  True label : graphics} \newline 
    I am looking for EISA or VESA local bus graphic cards that support at least  
|1024x786x24 resolution.  I know Matrox has one, but it is very 
|expensive. All the other cards I know of, that support that 
|resoultion, are striaght ISA.  
What about the ELSA WINNER4000 (S3 928, Bt485, 4MB, EISA), or the 
Metheus Premier-4VL (S3 928, Bt485, 4MB, ISA/VL) ? 
|Also are there any X servers for a unix PC that support 24 bits? 
As it just happens, SGCS has a Xserver (X386 1.4) that does 
1024x768x24 on those cards. Please email to info@sgcs.com for more details. 
- Thomas   &VESA, PC, looking, 24, unix, email, resolution, graphic, info, support  & cards, resoultion, support, unix, bus, details, com, bits, Metheus, servers   & expensive, cards, support   & support, details, expensive, cards, bits, servers, info, Premier, resolution, bus  & Bus, Premier, graphic, Metheus, support, unix, expensive, details, ELSA, com   \\
\hline
    \end{tabular}
    \caption{Sample Explanations for correctly classified  documents in the IMDB and Newsgroup Dataset.}
    \label{tab:textexplanations}
\end{table*}

\begin{table*}[]
\small
    \begin{tabular}{|p{6cm}|p{1.5cm}|p{1.5cm}|p{1.5cm}|p{1.5cm}|p{1.5cm}|}
    \hline
    Review/Document & LIME & SHAP & Anchors & L2X & MAIRE\\
\hline
\textbf{model prediction : negative, True label : positive} \newline
this movie gets a 10 because there is a lot of gore in it.who cares about the plot or the acting.this is an Italian horror movie people so you know you can't expect much from the acting or the plot.everybody knows fulci took footage from other movies and added it to this one.since i never seen any of the movies that he took footage from it didn't matter to me.the Italian godfather of gore out done himself with this movie.this is one of the goriest Italian movies you will ever see.no gore hound should be without this movie in their horror movie collection.buy this movie no matter what it is a horehound's dream come true.  & Plot, acting, didn, true, horror, collection, dream, footage, gets, movie  & horror, cares, goriest, footage, never, expect, hound, fulci, matter, plot  & horror, plot, hounds, matter  & True, cares, collection, dream, never, acting, gore, fulci, matter, expect  & Matter, acting, horror, dream, plot, movie, footage, collection, matter, cares  \\
    \hline
    \textbf{model prediction : christian, True label : atheism} \newline
    Pardon me if this is the wrong newsgroup.  I would describe myself as an agnostic, in so far as I'm sure there is no single, universal supreme being, but if there is one and it is just, we will surely be judged on whether we lived good lives, striving to achieve that goodness that is within the power of each of us.  Now, the complication is that one of my best friends has become very fundamentalist.  That would normally be a non-issue with me, but he  feels it is his responsibility to proselytize me (which I guess it is, according to his faith).  This is a great strain to our friendship.  I would have no problem if the subject didn't come up, but when it does, the discussion quickly begins to offend both of us: he is offended because I call into question his bedrock beliefs; I am offended by what I feel is a subscription to superstition, rationalized by such circular arguments as 'the Bible is God's word because He tells us in the Bible that it is so.'  So my question is, how can I convince him that this is a subject better left undiscussed, so we can preserve what is (in all areas other than religious beliefs) a great friendship?  How do I convince him that I am 'beyond saving' so he won't try?  Thanks for any advice.  & Bible, faith, beliefs, just, religious, good, word, feel, lives, Thanks  & universal, lives, faith, power, religious, offend, superstition, convince, beliefs, complication  & Superstition, Bible, faith  & Religious, responsibility, complication, universal, subject, strain, power, circular, Subscription, convince & Strain, faith, God, superstition, saving, great, good, beliefs, religious, Bible  \\
    \hline
    \end{tabular}
    \caption{Sample Explanations for incorrectly classified  documents in the IMDB and Newsgroup Dataset.}
    \label{tab:textmisclassifiedexplanations}
\end{table*}











\newpage
\bibliographystyle{elsarticle-num}
\bibliography{aij.bib}


\end{document}